%% file: main.tex
\documentclass{ecai} 

\input{packages}
\input{macros}

\newcommand{\BibTeX}{B\kern-.05em{\sc i\kern-.025em b}\kern-.08em\TeX}

\begin{document}


\begin{frontmatter}


\paperid{2069} 


\title{Rigorous Probabilistic Guarantees \\for Robust Counterfactual Explanations}


\author[A]{\fnms{Luca}~\snm{Marzari}\footnote{Equal contribution.\\ Contact authors: luca.marzari@univr.it, f.leofante@imperial.ac.uk \\ Code:\url{https://github.com/lmarza/APAS} \\ This is the full version of the paper by the same title appearing at ECAI 2024. This version includes proofs and additional experimental details.}}
\author[B]{\fnms{Francesco}~\snm{Leofante}\footnotemark}
\author[A]{\fnms{Ferdinando}~\snm{Cicalese}} 
\author[A]{\fnms{Alessandro}~\snm{Farinelli}}

\address[A]{Department of Computer Science, University of Verona, Italy}
\address[B]{Department of Computing, Imperial College London, United Kingdom}


\begin{abstract}
%
    We study the problem of assessing the robustness of counterfactual explanations for deep learning models. We focus on \textit{plausible model shifts} altering model parameters and propose a novel framework to reason about the robustness property in this setting. To motivate our solution, we begin by showing for the first time that computing the robustness of counterfactuals with respect to plausible model shifts is NP-complete. As this (practically) rules out the existence of scalable algorithms for exactly computing robustness, we propose a novel probabilistic approach which is able to provide tight estimates of robustness with strong guarantees while preserving scalability. Remarkably, and differently from existing solutions targeting plausible model shifts, our approach does not impose requirements on the network to be analyzed, thus enabling robustness analysis on a wider range of architectures. Experiments on four binary classification datasets indicate that our method improves the state of the art in generating robust explanations, outperforming existing methods on a range of metrics.
\end{abstract}

\end{frontmatter}

\input{sections/introduction}

\input{sections/related}
\input{sections/background}
\input{sections/hardness}

\input{sections/compare}

\input{sections/approx_sol}
\input{sections/evaluation}

\input{sections/conclusion}

\bibliography{mybibfile}

\clearpage
\input{sections/appendix-proofs}

\end{document}

%% file: packages.tex
\usepackage{amsmath}
\usepackage{amsthm}
\usepackage{amsfonts}
\usepackage{hyperref}
\usepackage{float}
\usepackage{graphicx}
\usepackage{tikz}
\usepackage{stackengine}
\usepackage{algorithm}
\usepackage{algorithmic}
\usepackage{todonotes}
\usepackage{bm}
\usepackage{rwthcolours}
\usepackage{subcaption}
\usepackage{booktabs}

\usepackage{pgfplots}
\pgfplotsset{compat=1.10}
\usetikzlibrary{patterns}
\usetikzlibrary{shapes.misc}
\tikzset{cross/.style={cross out, draw=black, minimum size=2*(#1-\pgflinewidth), inner sep=0pt, outer sep=0pt},
cross/.default={2pt}}

\usepackage{tcolorbox}

%% file: macros.tex
\newcommand{\model}[1]{\mathcal{M}_{#1}}

\newcommand{\intervalnet}{\mathcal{I}}

\DeclareMathOperator*{\argmin}{arg\,min}

\newcommand{\mshift}{S}
\newcommand{\dist}[3]{d_{#1}(#2,#3)}
\newcommand{\distance}[3]{\lVert #1 - #2\rVert_{#3}}

\newcommand{\ourmethod}{\texttt{AP$\Delta$S}~}

\newcommand{\R}{\mathbb{R}}
\newcommand{\name}{\Delta}

\newcommand{\relu}{\text{ReLU}}

 \newcommand{\relua}{\draw[line width=1.5pt] (-1em,0) -- (0,0)
                                (0,0) -- (0.75em,0.75em);}

\newtheorem{theorem}{Theorem}

\newtheorem{proposition}[theorem]{Proposition}
\newtheorem{lemma}[theorem]{Lemma}

\newtheorem{definition}{Definition}

%% file: sections/introduction.tex
\section{Introduction}

Understanding and interpreting the decisions of black-box deep learning models has become a dominant goal of Explainable AI (XAI). Several strategies have been proposed to this end. In this paper, we focus on counterfactual explanations (CFX) (see~\cite{StepinACP21,KarimiBSV23} for recent surveys on the topic), which aim to 
demystify the decision-making of a Deep Neural Network (DNN) by showing how an input needs to be changed to yield a different, typically more desirable, decision. Consider the widely used example of a loan application, where a mortgage applicant represented by an input $x$ with features \emph{unemployed} status, $25$ years of age, and \emph{low} credit rating applies for a loan and is rejected by the bank's AI. A CFX for this decision could be a slightly modified input, where increasing credit rating to \emph{medium} would result in the loan being granted. 

As CFXs have the potential to influence decisions with strong societal implications, their reliability has become the subject of intensive study (see~\cite{Jiang2024survey} for a survey). In particular, recent work has highlighted issues related to the robustness of CFXs against \emph{Plausible Model Shifts} (PMS)~\cite{UpadhyayJL21,Jiang_Leofante_Rago_Toni_2023}, showing that the validity of CFXs is likely to be compromised when bounded perturbations are applied to the parameters of a DNN, e.g., as a result of fine-tuning~\cite{UpadhyayJL21,BlackWF22,nguyen2022robust,Jiang_Leofante_Rago_Toni_2023,Hammanetal23}. Consider the loan example: if retraining occurs while the applicant is working toward improving their credit rating, without robustness, their modified case may still result in a rejected application, leaving the bank liable due to their conflicting statements.

In this paper, we focus on this troubling phenomenon and advance the state of the art in CFX robustness research in several directions.
More specifically, in §~\ref{sec:hardness}, we study the computational complexity of exactly determining whether a CFX is robust to PMS, formally showing for the first time that answering this question is NP-complete. As our result rules out the existence of practical algorithms to exactly compute the CFX robustness, we shift our attention toward methods for obtaining probabilistic guarantees on the robustness of CFX under model shifts. We, therefore, consider the work by~\citet{Hammanetal23}, where they propose a probabilistic approach to compute the robustness of CFX under \emph{Naturally-Occurring Model Shifts} (NOMS). Even though both PMS and NOMS notions are commonly used in the literature, very little is known about their potential interplay, and the question of whether robustness to NOMS subsumes robustness to PMS is still unresolved. We settle this question in §~\ref{sec:compare}, where we show that these two notions capture profoundly different scenarios, proving that robustness guarantees given for NOMS do not directly extend to PMS. Having settled this, in §~\ref{sec:main_results} we present \ourmethod, a novel sampling-based certification algorithm which allows to determine a provable probabilistic bound on the maximum shift a CFX can tolerate under PMS. Unlike existing solutions for robustness under PMS, our approach comes with significantly reduced computational requirements and does not make any assumption on the underlying DNN, thus making it applicable to a wider range of architectures. Finally, to assess the effectiveness of our proposed solution, in §~\ref{sec:experiments}, we study the general performance of our certification algorithm and provide a comprehensive comparison of the proposed approach against several state-of-the-art methodologies for CFX robustness certification. Crucially, we show that our approach can also be used to generate robust CFXs, outperforming existing methods on a number of metrics from the CFX literature.

%% file: sections/related.tex
\section{Related Work}

Various methods for generating CFXs for DNNs have been proposed. The seminal work of \cite{Wachter_17} framed the task of generating CFXs as a gradient-based optimization problem and proposed a loss that promotes CFX \emph{validity} (i.e., the CFX successfully changes the classification outcome of the network) and \emph{proximity} (i.e., the CFX is as close as possible to the original input for some distance metric). In addition to these metrics, other important properties have been highlighted as crucial for the practical applicability of CFXs. Prominent examples include \emph{plausiblity} (i.e., the CFX must lie on the data manifold)~\cite{poyiadzi2020face,PawelczykBK20} and \emph{actionability} (i.e., the changes suggested by the CFX must be achievable by the user in practice)~\cite{ustun2019actionable}.
Differently from these works, here we focus on the robustness property of CFXs.

Several forms of CFX robustness have been studied in the literature~\cite{Jiang2024survey}. Robustness to input changes is the focus of, e.g.~\cite{Slack_21,Dominguez-Olmedo_22,ZhangCWL23,LeofantePotyka24}, where solutions are devised to ensure that explanation algorithms return similar CFXs for similar inputs. In another line of work, \cite{GuyomardFGBT23,Virgolin_23,LeofanteL23,PawelczykDHKL23} considered the problem of generating adversarially robust CFXs that preserve validity under imperfect (or noisy) execution. Robustness to model multiplicity is instead considered in, e.g.~\cite{PawelczykBK20,LeofanteBR23,jiang2023recourse}, where  CFXs that preserve validity across sets of models are sought. However, the study of these forms
of robustness is outside the scope of this paper as our focus is on model shifts. Robustness to model shift has been studied in, e.g.~\cite{UpadhyayJL21,BlackWF22,nguyen2022robust,Jiang_Leofante_Rago_Toni_2023,Hammanetal23}. Of these, the approaches of~\cite{UpadhyayJL21} and~\cite{Jiang_Leofante_Rago_Toni_2023} are the most closely related to our work. The former presents an approach to generate robust CFXs under PMS using techniques from continuous optimization, which is able to guarantee robustness in the average-case scenario. The latter instead solves the same problem using abstraction techniques and discrete optimization tools, obtaining robustness guarantees that hold under worst-case conditions. Given their relevance, both approaches will be considered for an extensive experimental comparison in §~\ref{sec:experiments}.

%% file: sections/background.tex
\section{Background}\label{sec:background}

\paragraph{\textbf{(Neural) Classification model}.} Let $\mathcal{X} \subseteq \R^d$ denote the input space of a \emph{classifier} $\model{\theta}: \mathcal{X} \rightarrow [0,1]$ mapping an input $x \in \mathcal{X}$ to an output probability between $0$ and $1$. We consider classifiers implemented by feed-forward DNNs parameterized by a \emph{(parameter) vector} $\theta \in \Theta \subseteq \R^k$. Given two parameter vectors $\theta, \theta' \in \Theta$, we refer to the corresponding classifiers $\model{\theta}$ and $\model{\theta'}$ as \emph{instantiations} of the same parametric classifier $\model{\Theta}$. We assume concrete valuations of $\theta$ are learned from a set of labeled inputs as customary in supervised learning settings~\cite{Goodfellowetal}. Once $\theta$ has been learned, the classifier can be used for inference. Without any loss of generality, we focus on binary classification tasks, i.e., the classification decision produced by $\model{\theta}$ for an unlabeled input $x$ is $1$ if $\model{{\theta}}(x) \geq 0.5$, and $0$ otherwise. 

\paragraph{\textbf{Counterfactual explanations.}}
Existing methods in the literature define CFXs as follows. 

\begin{definition}\label{def:cfx} 
Consider an input $x \in \mathcal{X}$ and a classifier $\model{{\theta}}$ s.t. $\model{{\theta}}(x) < 0.5$. Given a distance metric $d: \mathcal{X} \times \mathcal{X} \rightarrow \mathbb{R}^+$, a \emph{(valid) counterfactual explanation} is any $x'$ such that:
\begin{subequations}
\begin{alignat*}{2}
&\argmin_{x'\in \mathcal{X}}  && d(x,x') \\ 
&\text{subject to} && \quad \model{{\theta}}(x') \geq 0.5 
\end{alignat*}
\end{subequations}
\end{definition}

Intuitively, given an input $x$ for which the classifier produces a negative outcome, a counterfactual explanation is a new input $x'$ which is similar to $x$, e.g., in terms of some specified distance between features values, and for which the classifier predicts a different outcome. Common choices for $d$ include the $\ell_1$ and $\ell_{\infty}$ norms~\cite{Wachter_17}, which will also be used in this work.

\paragraph{\textbf{Robustness to model shifts}}
Among several notions of robustness, recent work has placed emphasis on generating CFXs that remain valid under (slight) shifts in the classifier they were generated for. While existing approaches rely on a diverse range of techniques to solve this problem, they all share a common understanding of what constitutes a model shift, which we present next.
 
 \begin{definition}[\citet{Jiang_Leofante_Rago_Toni_2023}]
Let $\model{\theta}$ and $\model{\theta'}$ be two instantiations of a parametric classifier $\model{\Theta}$. 
For $0 \leq p  \leq \infty$, the \emph{p-distance between $\model{\theta}$ and $\model{\theta'}$} is defined as $\dist{p}{\model{\theta}}{\model{\theta'}} = \distance{\theta}{\theta'}{p} $.
\label{def:distance_between_models}
\end{definition}

\begin{definition}[\citet{Jiang_Leofante_Rago_Toni_2023}]
A \emph{model shift} (w.r.t. a fixed $p$-distance) is a function $\mshift $ mapping a classifier $\model{\theta}$ into another classifier $\model{\theta'} = \mshift(\model{\theta})$ such that:   
\begin{itemize}
    \item $\model{\theta}$ and $\model{\theta'}$ are instantiations of the same 
    $\model{\Theta}$;
    \item $\dist{p}{\model{\theta}}{\model{\theta'}} > 0$. 
\end{itemize}
\label{def:model_shift}
\end{definition} 

Informally, a model shift captures changes in the parameters of a DNN, but does not affect its architecture. Based on this definition, we can formalize the robustness property for a CFX as follows.

\begin{definition}
    Consider an input $x \in \mathcal{X}$ and a classifier $\model{{\theta}}$ s.t. $\model{{\theta}}(x) < 0.5$. Let $x'$ be a counterfactual explanation computed for $x$ s.t. $\model{{\theta}}(x') \geq 0.5$ . Given a set of model shifts $\Delta$, we say that the counterfactual $x'$ is \emph{$\Delta$-robust} if $\mshift(\model{\theta})(x') \geq 0.5$ for all $\mshift \in \Delta$.
\end{definition}

The definition of a model shift can be specialized to better characterize how $\theta$ is allowed to change under $\mshift$. In the following, we report two most commonly studied notions of model shifts: \emph{Naturally-Occurring Model Shifts} and \emph{Plausible Model Shifts}. 

\begin{definition}[ \citet{Hammanetal23} (NOMS)]
\label{def:NOMC}
Consider a classifier $\model{{\theta}}.$ A set of model shifts $\Delta$ is said to be \emph{naturally occurring} if for a (randomly) chosen model shift $S$ from $\Delta$ and $\model{{\theta'}} = S(\model{{\theta}})$ being the new classifier obtained after applying  $\mshift$ to $\model{{\theta}}$ the following hold: 
\begin{itemize}
    \item $\mathbb{E}[\model{{\theta'}}(x)] = \model{{\theta}}(x)$; where the expectation is over the randomness of $\model{{\theta'}}$ given a fixed value of $x$;
    \item $\text{Var}[\model{{\theta'}}(x)] = \nu_x$, where $\nu_x$ represents the maximum variance of the prediction of $\model{{\theta'}}(x)$, and whenever $x$ lies on the data manifold $\mathcal{X}$,  $\nu_x$  is upper bounded by a small constant $\nu$;
    \item If $\model{{\theta}}$ is Lipschitz continuous for some $\gamma_1$, then $\model{{\theta'}}(x)$ is also Lipschitz continuous for some $\gamma_2$.
\end{itemize}
\end{definition}


Broadly speaking, a naturally-occurring model shift allows the application of arbitrary changes to $\theta$ as long as the overall behavior of the classifier is not affected. This is in contrast with the notion of plausible model shift~\cite{UpadhyayJL21,Jiang_Leofante_Rago_Toni_2023}, which requires changes to be bounded.  

\begin{definition}[\citet{Jiang_Leofante_Rago_Toni_2023} (PMS)]
\label{def:deltachange}
Consider a classifier $\model{{\theta}}$ and a new classifier $\model{{\theta'}} = S(\model{{\theta}})$ obtained after applying a model shift $\mshift$ to $\model{{\theta}}$. Given some $\delta \in \mathbb{R}_{>0}$ and $0 \leq p \leq \infty$, $\mshift$ is said to be \emph{plausible} (w.r.t.\ the choice of parameters $\delta$ and $p$)
if $\dist{p}{\model{\theta}}{\mshift(\model{\theta})} \leq \delta$.
\end{definition} 

Given an upper bound $\delta$, a set of PMS $\Delta$ can then be obtained by considering all shifts $\mshift$ that satisfy Definition~\ref{def:deltachange}, i.e.  $\Delta = \{ \mshift \mid \dist{p}{\model{\theta}}{\mshift(\model{\theta})} \leq \delta\}$. In the rest of this paper, we will sometimes use $\Delta_{\delta}$ to make the upper bound $\delta$ explicit. Given a set $\Delta$, a realization of $\Delta$ can be defined as follows.

\begin{definition}\label{def:delta_plausible}
   Given a classifier $\model{\theta},$ and a set of plausible model shifts $\Delta$, we say that a {\em realization of ${\Delta}$} is a classifier $\model{\theta'}$  
%
    such that $\theta' \in [\theta - \delta, \theta+\delta]$
\end{definition}



\citet{Jiang_Leofante_Rago_Toni_2023} proposed to reason about robustness under PMS using an Interval Neural Network (INN)~\cite{PrabhakarA19} as an intermediate representation. The interval weights of an INN allow to represent an over-approximation of all the possible models obtainable under a set of PMS $\Delta$, thus providing a compact representation of the problem. In this work, instead, we address the robustness of CFXs by reasoning directly in terms of the set $\Delta$. This results in several computational improvements as we will discuss in section \ref{sec:main_results}.

%% file: sections/hardness.tex
\section{Checking Robustness is Hard}\label{sec:hardness}

In this section, we study the computational complexity of deciding whether a given counterfactual explanation is robust in the presence of model shifts. Our aim here is to better understand the computational challenges arising from this problem and to use these results to guide the development of novel, more efficient certification procedures. Without loss of generality, we consider PMS to encode the problem.
Deciding whether a given CFX $x'$ is robust to a set of PMS $\Delta$ requires to check whether $\Delta$ contains at least one realization which yields a classification outcome that is different from the intended CFX outcome. For instance, in the case of the loan example, this would correspond to a model rejecting the loan ($\model{\theta}(x') < 0.5$) as opposed to accepting it as intended. This problem can be formulated as follows.

\begin{tcolorbox}
  \vspace{-0.1cm}

{\sc Distinct-Realizations Problem ({\sc DRP})}
\vspace{2mm}

{\bf Input}: a classifier $\model{\theta},$ a set $\Delta$ of PMS, an input $x$, and a threshold $\tau$.

\vspace{2mm}

{\bf Output}: yes $\iff \exists\; \model{\theta_1}, \model{\theta_2}$ $\subseteq {\Delta}$ s.t

\hspace{2.7cm} $\model{\theta_1}(x) < \tau \leq \model{\theta_2}(x)$.
 
  \vspace{-0.1cm}

\end{tcolorbox}



\begin{theorem}\label{th:hardness-CFX}
    Deciding {\sc DRP} is NP-complete.
\end{theorem}
\par
\noindent
{\em Proof sketch.}
    The inclusion of DRP in NP requires two forward propagations of $x$ through two concretizations (i.e., the certificates) checking if $\tau$ is between the two computed outputs. This is clearly polynomial in the size of the classifier. Regarding the hardness, we can show that {\sc 3-SAT} reduces to {\sc DRP}.
    Given a formula $\phi$ we can produce a DNN classifier $\model{\theta}$, an input $x$ and 
    a $\delta$ (maximum shift on the edge weights), defining a ${\Delta}$ such that $\phi$ is satisfiable 
    if and only if there exists another  $\model{\theta'}$ which is also a realization
    of ${\Delta}$ such that $0.5 \leq \model{\theta'}(x) $ and  $\model{\theta}(x) < 0.5.$


We build on the reduction of \cite{ReluplexJournal}, which, 
given a formula $\phi$ produces a neural network such that satisfying assignments for $\phi$ 
are encoded into inputs to the network producing a desired output. 
 
A first observation is that if we start with the neural network $\model{\theta}$ produced by the reduction of \cite{ReluplexJournal}, and replace any weight equal to 1 with the interval $[1-2\delta, 1]$ and  
any weight -1 with the interval $[-1-2\delta, -1],$ we obtain $\Delta$ such that for any $\model{\theta}'$ being a realization of $\Delta$ and any input $x,$
$\model{\theta'}(x) \in [(1-p_1(\delta))\model{\theta}(x), (1+p_2(\delta))\model{\theta}(x)],$
where $p_1, p_2$ are some fixed polynomials defined by the number of layers in $\model{\theta}$. Therefore, we can extend the hardness of \cite{ReluplexJournal} to the case of plausible model shift, where instead of checking for the output to be $t$, we could use an additional layer to check whether the output is in the interval  $[(1-p_1(\delta))t, (1+p_2(\delta))t]$. 

However, with respect to the hardness proof of \cite{ReluplexJournal}, there is a substantial difference in our problem since, besides starting from a DNN with each weight in $[\theta_i-\delta, \theta_i+\delta]$, we want to map satisfying assignments to realizations of the $\Delta$, rather than inputs of the network. In fact, encoding assignments to choices of the weights turns out to require significantly more.

Our reduction is designed so that the assignments to the variables of the CNF are encoded to the choices of some specific weights of the DNN, henceforth referred to as the {\em network edge main inputs}. These are the weights that determine the outputs of a gadget that we call the generating gadget. A generating gadget has a fixed input (representing the CFX), uses only intervals of width $2\cdot\delta$, and produces a value in $[0,1]$ where the Boolean $false$ is represented by a value close to 0 and the Boolean $true$ is represented by a value close to 1. 

The output of the generating gadget (one per each variable of the formula $\phi$) is sent both to the network simulating the CNF formula 
(as in \cite{ReluplexJournal}) and to a further gadget (the discretizer-gadget) that controls whether the output of the generating gadget is a discrete value in $\{0,1\}.$ 

The output of the network simulating the formula is then combined with the output of the discretizer-gadgets in such a way that the final output is $< 0.5$ whenever, either there is one of the outputs of the generating-gadgets (determined by the choice of the {\em network edge main input}) which is not in $\{0,1\}$ or the output of the subnetwork simulating the formula implies that the {\em network edge main input} encode and assignment not satisfying some clause of $\phi.$ The complete proof can be found in the supplementary material. \qed

From Theorem~\ref{th:hardness-CFX}, it follows that deciding whether a CFX $x'$ is not robust to a set of PMS $\Delta$ is NP-complete. This results lead to the following corollary.

\begin{theorem}\label{thm:density}
Given a model $\model{\theta}$ and a set of plausible model shifts $\Delta$, computing the number of model shifts in $\Delta$ for which a given CFX $x'$ is \emph{$\name$-robust} is NP-hard. 
\end{theorem}

Indeed, deciding non-$\Delta$-robustness and $\Delta$-robustness are equivalent under Turing reductions. Therefore, the counting problem for $\Delta$-robustness is at least as hard as the decision problem for non-$\Delta$-robustness. This hardness result motivates an approximate approach to estimate the robustness of a counterfactual under a set of PMS $\Delta$. 
   

%% file: sections/compare.tex
\section{Probabilistic Guarantees for Existing Notions of Model Shifts} \label{sec:compare}

As we have established in the previous section, exact methods for computing robustness under model shifts are bound to lack scalability. This motivates the design of probabilistic approaches to solve the problem. Previous work by~\citet{Hammanetal23} presented an approach to obtain counterfactual explanations that are probabilistically robust under NOMS. A natural question that arises then is whether guarantees obtained for NOMS also transfer to the PMS setting. As we show for the first time below, this is not the case in general.

\begin{lemma}
\label{lem:diff}
    Naturally-occurring model shifts may not be Plausible, and vice-versa.
\end{lemma}
\begin{proof}
     Consider the DNN $\model{\theta}$ depicted in Fig.~\ref{fig:lemma} (a) with two input nodes, one hidden layer with two ReLU nodes\footnote{In this proof, we consider a DNN with only ReLU activation functions. However, we notice that it is possible to have a similar counterexample even with other activations, e.g. Tanh, Sigmoid.} and one single output. The parameters $\theta = [w_1, \dots, w_6]$ are the weights on the edges listed 
     top-bottom and left-right.

    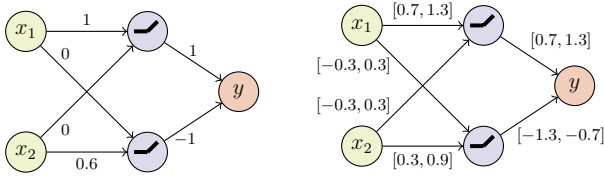
\begin{figure}[h!]
        
        \centering
        \begin{subfigure}[t]{0.5\columnwidth}
        \centering
        \input{imgs/network}
   
        \end{subfigure}%
        \begin{subfigure}[t]{0.5\columnwidth}
        \centering
        \input{imgs/intervals}
  
        \end{subfigure}%
         \vspace{3mm}
        \caption{(a) The model $\model{\theta}$ used as an example to prove the lemma. (b) An interval neural network representing the realizations that can be obtained from $\model{\theta}$ considering a set of PMS $\Delta_{\delta}$ with $\delta = 0.3$.}
        \label{fig:lemma}
         \vspace{5mm}
    \end{figure}

Propagating an input vector $x = [x_1, x_2]^T$ through $\model{\theta}$, we obtain $\model{\theta}(x) = y = w_5 \cdot \max\{0, w_1\cdot x_1 + x_2\cdot w_3\} + w_6 \cdot \max\{0, w_2\cdot x_1 + x_2 \cdot w_4\}$. Now assume an input vector $x = [0.9, 0.9]^T$ and weights $w_1=1, w_2=0,w_3=0,w_4=0.6,w_5=1, w_6=-1$. The corresponding output generated by the DNN is $\model{\theta}(x) = 0.46$.
A counterfactual for $x$ could be given as a new input vector $x' = [1, 0.8]^T$, for which we obtain $\model{\theta}(x')=0.52 > 0.5$. Now, following Definition~\ref{def:deltachange}, we consider a set of plausible model changes obtained for $\delta = 0.3$. This can be captured by defining on each weight $w_i$ the corresponding interval in $[w_i - \delta, w_i+\delta]$
depicted in Fig.~\ref{fig:lemma} (b) that represents the set of all the possible models obtained from $\model{\theta}$, replacing each $w_i$ with a weight in the interval $[w_i - \delta, w_i+\delta]$. We then have that, the expected result of a model $\model{\theta'}$ sampled uniformly from such a set satisfies:
\begin{align*}
    \mathbb{E}[\model{\theta'}(x')] =~&  \mathbb{E}[w_5]\cdot \mathbb{E}[\relu(x_1\cdot w_1 + x_2\cdot w_3)] ~+\\ &\mathbb{E}[w_6]\cdot \mathbb{E}[\relu(x_1\cdot w_2 + x_2\cdot w_4)] \\
    =~ & \mathbb{E}[[0.7,1.3]] \cdot \mathbb{E}[\max\{0, x_1\cdot[0.7, 1.3] ~+ \\ &x_2\cdot[-0.3, 0.3]\}] + \mathbb{E}[[-1.7,-0.3]] \cdot \\
    & \mathbb{E}[\max\{0, x_1\cdot[-0.3, 0.3] + x_2\cdot[0.3, 0.9]\}]\\
    & >  0.52 \neq \model{\theta}(x')
\end{align*}

Definition~\ref{def:NOMC} states that a model shift is naturally occurring if $\mathbb{E}[\model{\theta'}(x)]=\model{\theta}(x)$. This implies that ${\Delta}$ contains models that cannot be characterized as naturally occurring model changes. Vice versa, the existence of Naturally occurring model shifts not being plausible is implicit in the definition (we defer an example to the supplementary material) thus giving our result. 


\end{proof}

Lemma~\ref{lem:diff} shows the existence of witnesses proving that Definition~\ref{def:NOMC} (NOMS) and Definition~\ref{def:deltachange} (PMS) may capture very different model changes in general. To complement this observation, we also ran experiments to determine how often these definitions disagree empirically. Our results, reported in the Appendix, confirm that the two notions indeed capture two different settings in general. These results show that (probabilistic) methods devised for NOMS may fail to guarantee robustness under PMS, thus motivating the development of dedicated approaches for probabilistic guarantees under PMS. Indeed, having clarified the relationship between the two notions of model shifts, in the following, we focus on certification approaches for robustness under PMS, presenting a novel approximate solution with probabilistic guarantees.

%% file: imgs/network.tex
\begin{tikzpicture}[scale=0.8, every node/.style={scale=0.7}]

  
  \node[circle,draw=black, minimum width=0.75cm,fill=maygreen25] (input_1) at (0,0) {\Large$x_1$};
  \node[circle,draw=black, minimum width=0.75cm,fill=maygreen25] (input_2) at (0,-2) {\Large$x_2$};
  
  \node[circle,draw=black, minimum width=0.75cm,fill=lila25] (hidden_1) at (2,0) {};
  \begin{scope}[xshift=2cm,scale=0.7]
        \relua
    \end{scope}
  
  \node[circle,draw=black, minimum width=0.75cm,fill=lila25] (hidden_2) at (2,-2) {};
  \begin{scope}[xshift=2cm, yshift=-2cm,scale=0.7]
        \relua
    \end{scope}
  
    \node[circle,draw=black, minimum width=0.75cm,fill=red25] (output_1) at (3.5,-1) 
  {\Large$y$};

  \draw[->] (input_1) edge node[above]{{$1$}} (hidden_1);
  \draw[->] (input_1) edge node[below, xshift=-0.4cm, yshift=-0.5cm]{$0$} 
  (hidden_2);
  
  \draw[->] (input_2) edge node[above, xshift=-0.4cm, yshift=0.5cm]{$0$} 
  (hidden_1);
  \draw[->] (input_2) edge node[below]{$0.6$} 
  (hidden_2);
  
  \draw[->] (hidden_1) edge node[above]{{$1$}} (output_1);

  \draw[->] (hidden_2) edge node[below, xshift=-0.15cm, yshift=-0.1cm]{$-1$}  (output_1); 
  
\end{tikzpicture}

%% file: imgs/intervals.tex
\begin{tikzpicture}[scale=0.8, every node/.style={scale=0.7}]

  
  \node[circle,draw=black, minimum width=0.75cm,fill=maygreen25] (input_1) at (0,0) {\Large$x_1$};
  \node[circle,draw=black, minimum width=0.75cm,fill=maygreen25] (input_2) at (0,-2) {\Large$x_2$};
  
  \node[circle,draw=black, minimum width=0.75cm,fill=lila25] (hidden_1) at (2,0) {};
  \begin{scope}[xshift=2cm,scale=0.7]
        \relua
    \end{scope}
  
  \node[circle,draw=black, minimum width=0.75cm,fill=lila25] (hidden_2) at (2,-2) {};
  \begin{scope}[xshift=2cm, yshift=-2cm,scale=0.7]
        \relua
    \end{scope}
  
    \node[circle,draw=black, minimum width=0.75cm,fill=red25] (output_1) at (3.5,-1) 
  {\Large$y$};

  \draw[->] (input_1) edge node[above]{\small{$[0.7,1.3]$}} (hidden_1);
  \draw[->] (input_1) edge node[below, xshift=-1.3cm, yshift=-0.1cm]{\small$[-0.3,0.3]$} 
  (hidden_2);
  
  \draw[->] (input_2) edge node[above, xshift=-1.3cm, yshift=0.1cm]{\small$[-0.3,0.3]$} 
  (hidden_1);
  \draw[->] (input_2) edge node[below]{\small$[0.3,0.9]$} 
  (hidden_2);
  
  \draw[->] (hidden_1) edge node[above,xshift=0.6cm]{{\small$[0.7,1.3]$}} (output_1);

  \draw[->] (hidden_2) edge node[below, xshift=0.6cm, yshift=-0.1cm]{\small$[-1.3,-0.7]$}  (output_1); 
  
\end{tikzpicture}

%% file: sections/approx_sol.tex
\section{Robustness under PMS with Probabilistic Guarantees} \label{sec:main_results}

\citet{Jiang_Leofante_Rago_Toni_2023} proposed to use INNs to enable a compact representation of a superset of the models that can be obtained by a perturbation of the starting model under a set $\Delta$. By exploiting an exact reachable set computation method, e.g., based on MILP \cite{MILP}, the authors could determine whether or not a CFX is robust under the chosen $\Delta$ via a single forward propagation of the CFX. However, in view of the NP-hardness of the problem discussed in the § \ref{sec:hardness} and the typical non-linear nature of the classifiers, it presents some computational limitations. 

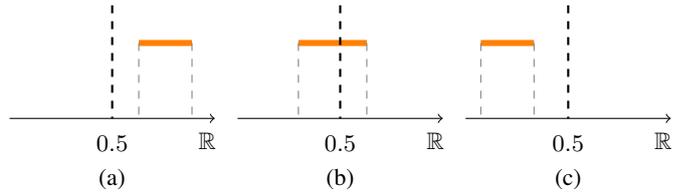
\begin{figure}[t]
    \centering
    \scalebox{1}{\input{imgs/intervals_sigmoid}}
     \vspace{3mm}
    \caption{Visual representation of the possible output reachable set for an interval abstraction for a binary classification model. (a) For a given ${\Delta}$, we classify an input as $1$ (robust) if the output range for that input is always greater $0.5$. Otherwise, the input is classified as $0$, i.e., not robust (b),(c).
    }
    \label{fig:inn_classfication}
     \vspace{5mm}
\end{figure}

In general, interval neural networks map inputs to intervals representing an over-approximation of all possible outcomes that can be produced by any shifted model $\model{\theta'}$ obtained under $\Delta$. Given this property, if 
the output reachable set is completely disjoint from the decision threshold $0.5,$ then one can assert -- in a sound and complete fashion -- whether or not a given CFX is robust (Fig.~\ref{fig:inn_classfication} (a,c)). On the other hand, if we run into a situation such as the one depicted in Fig.~\ref{fig:inn_classfication} (b), one cannot assert robustness with certainty. In this scenario, \citet{Jiang_Leofante_Rago_Toni_2023} propose to classify the CFX as not robust, which preserves the soundness of their result.  Nonetheless, this might lead to discarding a CFX even when the actual density of (equivalently, the probability that after retraining, we incur in) plausible model shifts for which the CFX is not robust is extremely low. As we will show in §~\ref{sec:experiments}, this worst-case notion of robustness affects the CFXs generated by~\cite{Jiang_Leofante_Rago_Toni_2023}, which may end up being unnecessarily expensive (in terms of proximity) and having low plausibility. Additionally, computing the exact output reachable set of an interval abstraction may be costly (e.g., MILP is known to be NP-hard). This is expected: Theorems \ref{th:hardness-CFX} and \ref{thm:density} show that there is no polynomial time algorithm able to return an exact estimate of the fraction of plausible shifts for which the  CFX is robust (hence a fortiori deciding whether it is $\Delta$-robust), unless P=NP. In the following, we propose a novel certification approach that aims to alleviate this problem.

\subsection{A Provable Probabilistic Approach}
\label{ssec:prob_approach}


One possible idea to avoid exact reachable set computation to determine the robustness of a CFX under PMS is to use naive interval propagation. Given an input CFX, we propagate this input through the network, keeping track of all the possible activation values that can be obtained under $\Delta$ until the output layer is reached. However, the non-linear and non-convex nature of DNNs may result in a significant overestimation of the actual reachable set, thus resulting in a spurious decision of non-robustness. In such cases, a CFX may end up being labeled as non-robust even though the CFX is actually robust. Additionally, even with exact methods, a CFX may be discarded even though the fraction of plausible model shifts in $\Delta$ for which the CFX is not robust is negligible.

To avoid these problems, we propose an approximate certification approach based on Monte-Carlo sampling that draws sample realizations directly from $\Delta$ to obtain an underestimation of the space of possible classifications under PMS. The idea of using a sample-based approach stems from the fact that the $\Delta$ set, representing all the plausible model shifts, abstracts an infinite number of models to test. As testing this infinite number of models may be impossible in practice, efficient sampling-based solutions hold great promise. In detail, given a CFX $x'$ we can compute an underestimation of the output reachable set under $\Delta$ by sampling $n$ random realizations $\model{\theta_1}, \dots, \model{\theta_n}$ from ${\Delta},$ and compute the output reachable set by taking, respectively, the $\min_i \model{\theta_i}(x')$ and the $ \max_i \model{\theta_i}(x')$ for $i \in \{1,\ldots,n\}$. 

This approach is very effective and allows us to obtain an estimate of the output reachable set without using an exact solver. Nonetheless, the number $n$ of realization to sample in order to achieve a good reachable set estimation remains unclear, as well as what kind of guarantees one could obtain from this approach. To answer these questions, we leverage previous results on the \textit{statistical prediction of tolerance limits} \cite{wilks1942statistical,porter2019wilks,eProve}. Indeed, we observe that for each realization $\model{\theta_i}$ sampled from $\Delta$, the resulting output of the DNN $\mathcal{M}_{\theta_i}(x')$ can be interpreted as an instantiation of a random variable $X$ whose tolerance interval we are trying to estimate. Following this observation, we can derive a probabilistic bound on the correctness of the solution returned from $n$ samples, using the following lemma based on \cite{wilks1942statistical}:

\begin{lemma}\label{lemma:wilks}
    Fix an integer $n>0$ and an approximation parameter $R \in (0,1)$. Given a sample of $n$  models $\model{\theta_1}, \dots \model{\theta_n}$ from the (continuous) set of possible realizations $\Delta$, the probability that for at least a fraction $R$ of the models in a further possibly infinite sequence of samples $\model{\theta_1}^{(2)}, \dots \model{\theta_m}^{(2)}$ from $\Delta$  we have 
  \begin{equation} \label{condition:min}
    \min_i \model{\theta_i}^{(2)} (x) \geq \min_i \model{\theta_i}(x) 
    \end{equation}
  $$ \mbox{(respectively } \max_i \model{\theta_i}^{(2)} (x) \leq \max_i \model{\theta_i}(x) \mbox{)}$$
   is given by $\alpha = n \cdot \int_R^1 x^{n-1}\;dx = 1 - R^n$.
\end{lemma}

Informally, Lemma~\ref{lemma:wilks} allows us to derive the minimum number of realizations $n$ needed to guarantee that at least $R$ models within $\Delta$ satisfy the robustness property with probability $\alpha$.  Therefore, using these $n$ realizations, we can obtain an underestimation of the reachable set that is correct with confidence $\alpha$ for at least a fraction $R$ of indefinitely large further realizations of models from $\Delta$. 
In practice, if we set, e.g. $\alpha=0.999$ and $R=0.995$, we can derive $n$ as $n = \log_R (1-\alpha) = 1378$. After having selected $1378$ random realizations from $\Delta$, if the lower bound of the underestimated reachable set computed as $\min_i \model{\theta_i}(x')$ is greater than 0.5, then with probability $\alpha=0.999$, $R$ is a lower bound on the fraction of plausible model shifts in $\Delta$ for which $x'$ is robust. In other words, Lemma~\ref{lemma:wilks} allows us to assert with a confidence $\alpha$ that $x'$ is not $\Delta$-robust for at most a fraction $(1-R)=0.05$ of models from $\Delta$.

\subsection{The \ourmethod Algorithm}

Using the result of Lemma~\ref{lemma:wilks}, we now present our approximation method \ourmethod to generate probabilistic robustness guarantees. The procedure, shown in Algorithm~\ref{alg:new_approach}, receives as input a model $\model{\theta}$, a CFX $x'$ for which robustness guarantees are sought, 
and the two confidence parameters $\alpha, R$. The algorithm then searches for the 
largest $\delta_{max}$ such that the CFX $x'$ is robust for at least a fraction $R$ of the plausible model shifts up to $\delta_{max}$ with probability $\alpha$.

 \begin{algorithm}[t]
\caption{Approximate Plausible $\Delta$-Shift (\texttt{AP$\Delta$S})}\label{alg:new_approach}
\begin{algorithmic}[1]
\small
\STATE \textbf{Input:} Model $\model{\theta}$, set of PMS $\Delta$, CFX $x'$, $\alpha$, $R$ 
\STATE \textbf{Output: } $\delta_{max}$
\vspace{0.2cm}

\STATE $n \gets \log_R (1-\alpha)$ \hfill $\rhd$ number of samples \label{numreal}
\STATE $\delta_{init} \gets  0.0001$ 
\STATE rate $\gets \texttt{realizations}(\model{\theta}, x', \delta_{init}, n)$
 \IF{rate $\neq 1$}
    \STATE \textbf{return} $0$ \hfill $\rhd$ not robust for $\delta_{init}$
\ENDIF
\vspace{0.2cm}
\STATE $\delta \gets \delta_{init}$
\WHILE{rate $= 1$}
    \STATE $\delta \gets 2\delta$
    \STATE rate $\gets \texttt{realizations}(\model{\theta},  x', \delta, n)$
\ENDWHILE

\vspace{0.2cm}
$\rhd$ we exit from the while because we have found at least one model in the realizations
 with an output $< 0.5$, and we have $[\delta/2,\; \delta)$ to search for a $\delta_{max}$.
\STATE $\delta_{max} \gets \delta/2$

\WHILE{True}
    \IF{$\vert \delta - \delta_{max}\vert \leq \delta_{init}$}
        \STATE \textbf{return} $\delta_{max}$
    \ENDIF
    \STATE $\delta_{new} \gets (\delta_{max} + \delta)/2$
    \STATE rate $\gets \texttt{realizations}(\model{\theta}, x', \delta_{new}, n)$
    \IF{rate $= 1$}
        \STATE $\delta_{max} \gets \delta_{new}$
    \ELSE
        \STATE $\delta \gets \delta_{new}$
    \ENDIF
\ENDWHILE
\end{algorithmic}
\end{algorithm}
 
The algorithm starts by computing the size $n$ of a sample of realizations that is sufficient to guarantee the condition in Lemma~\ref{lemma:wilks}~(line~\ref{numreal}). \ourmethod then initializes a small $\delta_{init}$ and checks if $x'$ is at least robust to a small model shift. To this end, it employs $\texttt{realizations}(\model{\theta}, x', \delta, n)$ which samples $n$ realizations, pertubating each model parameter by at most a factor $\delta$ and checks if for each of these realization $\model{\theta_i}(x') \geq 0.5$, thus computing a robustness rate. 
If not all these realizations result in a robust outcome, thus achieving a final rate not equal to 1, the algorithm discards the CFX $x'$ as non-robust (lines 6-8). Otherwise, it combines an exponential search (lines 9-12) and a binary search (lines 13-24) to find $\delta_{\max}$. At each step of this search, the procedure checks whether for each of the $n$ realizations from $\Delta = \{ \mshift \mid \dist{p}{\model{\theta}}{\mshift(\model{\theta})} \leq \delta_{max}\}$ holds $\model{\theta_i}(x') \geq 0.5$. 

\begin{proposition} \label{th2_eprove}
Given a model $\model{\theta}$ and a CFX $x'$, let $\delta^*$ be the (exact) maximum magnitude of model shifts such that 
$x'$ is robust with respect to the set of PMS $\Delta_{\delta^*}.$ Then, with probability $\alpha,$ \ourmethod returns a $\delta_{\max} \geq \delta^*$ such that the CFX $x'$ is robust for at least a fraction $R$ of the set of PMS $\Delta_{\delta_{\max}}.$ Moreover, the computation of $\delta_{max}$ is polynomial.
\end{proposition}

\par
\noindent
{\em Proof sketch.}
The $\delta_{\max}$ returned by the algorithm is obtained by iteratively increasing $\delta$, sampling $n$ models from the corresponding $\Delta_{\delta}$ and verifying that $\model{\theta_i}(x) \geq 0.5$ for each model $\model{\theta_i}$ sampled. By definition, $\delta^*$ is the actual value we are trying to estimate. Therefore, for each $\delta \leq \delta^*$, we will always obtain realizations $\model{\theta_i}$ for which $\model{\theta_i}(x') \geq 0.5$. Therefore, \ourmethod will always return $\delta_{\max}$ values that are at least equal to $\delta^*$.
Once the exponential search ends, by exploiting Lemma \ref{lemma:wilks}, we can state that with probability $\alpha$, the CFX $x'$ is robust for at least $R$ of any infinite further realizations from $\Delta_{\delta_{max}}$. The time complexity of the algorithm corresponds to $n \cdot m$ forward propagations, with $n$ being the sample size and $m = \log \frac{\delta_{\max}}{\delta_{init}}$ being the number of iterations of the exponential search, which is polynomial in the input size of the problem. \qed


%% file: imgs/intervals_sigmoid.tex
\begin{tikzpicture}

 
  
 

  

  

\draw[line width=0.8mm, color=orange] (4,0) -- (4.7,0);
\draw[dashed,gray] (4,0) -- (4,-1);
\draw[dashed,gray] (4.7,0) -- (4.7,-1);

\draw[dashed,black, thick] (3.65,0.5) edge node[above,yshift=-1.3cm]{\color{black}$0.5$} (3.65,-1);

\draw[->] (2.3,-1) -- (5,-1);
\node[] (phantom_r1) at (4.9,-1.3) {$\R$};

\node[] (phantom_1) at (3.65,-1.8) {(a)};


\draw[line width=0.8mm, color=orange] (6.1,0) -- (7.0,0);
\draw[dashed,gray] (6.1,0) -- (6.1,-1);
\draw[dashed,gray] (7.,0) -- (7.,-1);

\draw[dashed,black, thick] (6.65,0.5) edge node[above,yshift=-1.3cm]{\color{black}$0.5$} (6.65,-1);

\draw[->] (5.3,-1) -- (8,-1);
\node[] (phantom_r2) at (7.9,-1.3) {$\R$};

\node[] (phantom_1) at (6.65,-1.8) {(b)};


\draw[dashed,black, thick] (9.65,0.5) edge node[above,yshift=-1.3cm]{\color{black}$0.5$} (9.65,-1);

\draw[line width=0.8mm,color=orange] (8.5,0) -- (9.2,0);
\draw[dashed,gray] (8.5,0) -- (8.5,-1);
\draw[dashed,gray] (9.2,0) -- (9.2,-1);

\draw[->] (8.3,-1) -- (11,-1);
\node[] (phantom_r3) at (10.9,-1.3) {$\R$};

\node[] (phantom_1) at (9.65,-1.8) {(c)};

\end{tikzpicture}

%% file: sections/evaluation.tex
\section{Experimental Analysis}\label{sec:experiments}

Section~\ref{sec:main_results} laid the theoretical foundations of a novel sampling-based method that allows to obtain provable probabilistic guarantees on the robustness of CFXs. In this section, we evaluate our approach by considering three experiments:
\begin{itemize}
    \item In §~\ref{ssec:soundness} we show how to instantiate \ourmethod in practice using a synthetic example.
    Specifically, we first demonstrate the interplay of parameters $n$, $\alpha$ and $R$ used to obtain a probabilistic guarantee. Then, using the maximum $\delta_{max}$ discovered by \ourmethod, we formally enumerate the number of models inside a set of PMS $\Delta_{\delta_{max}}$ for which a given CFX $x'$ is not robust. We show that this percentage is at most a fraction $(1-R)$, empirically confirming our theoretical results.
    \item In §~\ref{ssec:different_angles} we compare our certification approach with the one proposed in~\cite{Jiang_Leofante_Rago_Toni_2023}. In particular, we focus on the difference between the worst-case guarantees offered by their approach and compare them with the average-case guarantees of \ourmethod in terms of maximum shifts that can be certified. These experiments confirm our intuition that worst-case guarantees might be too conservative in practice, leading to a larger number of CFXs being discarded.
    \item Finally, in §~\ref{ssec:generating}, we consider the problem of generating robust CFXs and compare with two state-of-the-art approaches for robustness under PMS,~\cite{Jiang_Leofante_Rago_Toni_2023} and~\cite{UpadhyayJL21}. We show that our approach produces CFXs that are less expensive (in terms of $\ell_1$ distance) and more plausible, without sacrificing robustness. 
\end{itemize}

\subsection{\ourmethod in Action}
\label{ssec:soundness}

This experiment is designed to demonstrate how the three main parameters of \ourmethod, i.e. $n, \alpha$ and $R$, can be used to obtain probabilistic guarantees of robustness. To this end, we focus on the synthetic example depicted in Fig.~\ref{fig:enum}. Weights for the original network $\model{\theta}$, as well as the input used for testing robustness, are generated randomly.

 \begin{figure}[h!]
        \centering
        \scalebox{1}{\input{imgs/INN}}
        \vspace{3mm}
        \caption{The interval neural network used for exact enumeration.
        }
        \label{fig:enum}
         \vspace{5mm}
    \end{figure}
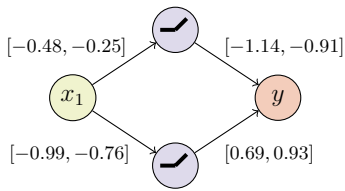
 
Considering a random input $x = -2.57$, we use \ourmethod to estimate a $\delta_{max}$ for which we seek the guarantee that for at least $R=90\%$ of the plausible model shifts induced by such $\delta_{max}$ the CFX $x'$ is robust. 
Following Proposition \ref{th2_eprove}, we set a confidence level $\alpha > 1- 10^{-40}$  (i.e., with certainty, in practice), which yields $n=100k$ realizations. For this setting, \ourmethod identifies a $\delta_{max}=0.115$. 

To validate this result, we define a procedure to exactly enumerate the models within $\Delta_{\delta_{max}}$ for which the robustness property does not hold. The interval abstraction proposed by \cite{Jiang_Leofante_Rago_Toni_2023} can be used to exactly enumerate the portion of the model shifts from $\Delta$ for which a CFX $x'$ is not robust. In fact, it is possible to build an interval neural network using the $\delta_{max}$ value identified by \ourmethod, setting each weight $w_i$ in $\theta$ to $[w_i-\delta_{max}, w_i+\delta_{max}]$. Then, recursively splitting each interval weight of the network in half allows to identify portions of $\Delta$ that are not robust. Employing this exact enumeration strategy (see Algorithm~3 in the Appendix for a complete formalization), after $s=7$ splits, we obtain that for $\sim 92\%$ of sub-interval networks, the CFX is robust. The remaining $8\%$ produced an \textit{unknown} answer (i.e., the situation depicted in Fig. \ref{fig:inn_classfication}(b)) that would require further splits, corresponding to only ten nodes to explore in the next iteration. 
In the worst case, even considering all the remaining ten nodes left to explore as non-robust, we would still obtain a maximum percentage of non-robustness lower than the desired upper bound $(1-R)=10\%$, confirming that the guarantees produced by \ourmethod indeed hold.



\input{sections/table_results}

\subsection{Worst-case vs Average-case Guarantees}
\label{ssec:different_angles}

\begin{figure}[b]
    \centering
    \includegraphics[width=0.8\linewidth,trim={1.1cm 0.6cm 3cm 2.4cm},clip]{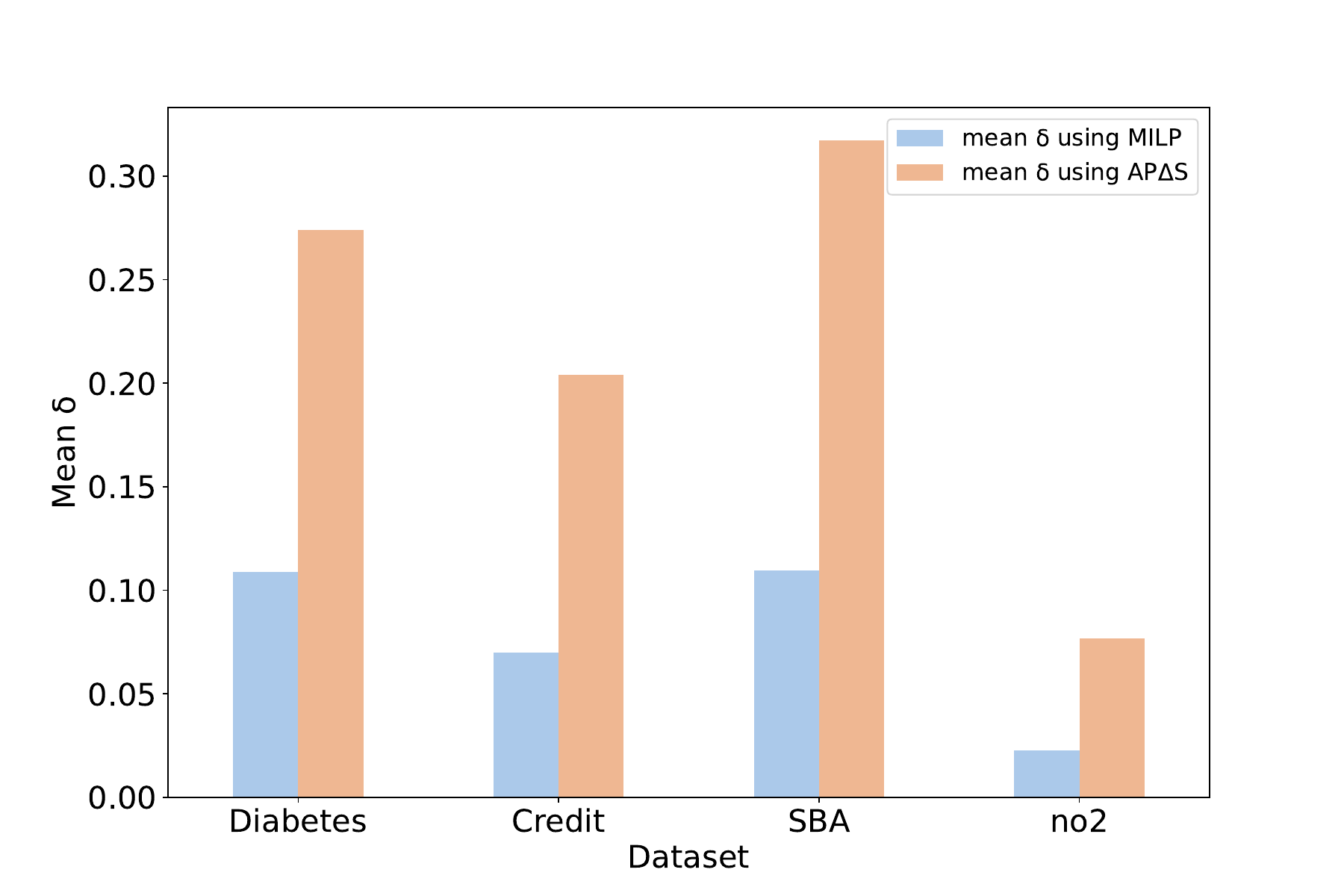}
     \vspace{3mm}
    \caption{Average robust $\delta$ obtained using MILP
    and \ourmethod.}
    \label{fig:mean_delta_comparison}
     \vspace{5mm}
\end{figure}

This set of experiments aims to compare the probabilistic guarantees offered by \ourmethod with the worst-case guarantees offered by~\cite{Jiang_Leofante_Rago_Toni_2023}. What we aim to show here is that adopting an average-case certification perspective may be more practical in some circumstances, as worst-case guarantees may be unnecessarily conservative. Our approach aims to obtain a $\delta_{max}$ for which the CFX is robust with confidence $\alpha$ for at least a fraction $R$ of model shifts in $\Delta$. This is in stark contrast with the worst-case reasoning of~\cite{Jiang_Leofante_Rago_Toni_2023}, where even a single realization of $\Delta$ for which the CFX is not robust results in the corresponding $\delta$ being discarded. 

To show why such strict guarantees may not be needed, we use an analogous experimental setup and the training process of~\cite{Jiang_Leofante_Rago_Toni_2023}, which considers four datasets: \textit{Diabetes} (continuous) \cite{smith1988using}, \textit{Credit} (heterogeneous) \cite{Dua2019}, \textit{no2} (continuous) \cite{OpenML2013} and \textit{Small Business Administration} (SBA) (continuous features) \cite{studentDS}. In particular, the Credit and the SBA datasets are known to contain distribution shifts~\cite{UpadhyayJL21} and are typically used to assess robustness under model changes. For the training procedure of the classifier, we randomly shuffle each dataset and split it into two halves, denoted $\mathcal{D}_1$ and $\mathcal{D}_2$. First, we use $\mathcal{D}_1$ to train a base neural network; then we use both $\mathcal{D}_1$ and $\mathcal{D}_2$ to train a shifted model. We then generate $50$ robust CFXs for the base network using the MILP-R and the same $\delta$ values as in~\cite{Jiang_Leofante_Rago_Toni_2023} for a fair comparison. Specifically, we use $\delta=0.11$ for \textit{Diabetes}, $\delta = 0.02$ for \textit{no2}, $\delta = 0.11$ for \textit{SBA} and $\delta = 0.05$ for \textit{Credit}. Subsequently, we evaluate the resulting CFXs by looking at two metrics: \textit{(i)} \textbf{VM1}, the percentage of CFXs that are valid on the base neural network and \textit{(ii)} \textbf{VM2}, the percentage of CFXs that remain valid for the shifted neural network trained using both $\mathcal{D}_1$ and $\mathcal{D}_2$. 

As previously observed by \citet{Jiang_Leofante_Rago_Toni_2023}, the training procedure used to generate shifted models may result in changes that exceed the $\delta$ used to generate provably robust CFXs. Indeed, after inspecting the networks obtained, we noted that the maximum empirical difference observed after retraining (denoted as $\delta_e$) is well above the $\delta$ values used during CFX generation. In particular, we recorded $\delta_e=0.27$ for \textit{Diabetes}, $\delta_e = 0.07$ for \textit{no2}, $\delta = 0.25$ for \textit{SBA} and $\delta_e = 1.28$ for \textit{Credit}. Given the magnitude of these shifts, 
the robustness of the CFXs generated by MILP-R cannot be guaranteed in practice. However, the results show a rather intriguing picture: the \textbf{VM2} metric appears to be unaffected by retraining, and all CFXs remain valid on the respective final models.

These results suggest that certification approaches based on worst-case reasoning may be too strict in practical scenarios. To further understand the implications of worst-case vs average-case reasoning, we adapted Algorithm~\ref{alg:new_approach} to use the certification procedure of Jiang et al., i.e., a MILP solver instead of a sampled-based approach, and compute the maximum provable $\delta^*$ for which the previously generated CFXs are robust (Algorithm~4 in the Appendix). Fig.~\ref{fig:mean_delta_comparison} shows a comparison between the average maximum provable $\delta$ obtained by this procedure and \ourmethod. As we can observe, our average-case guarantees allow to obtain $\delta$ values that are much higher, exceeding the MILP-certified in all instances. This is expected, given the results discussed in Proposition~\ref{th2_eprove}. However, what remains unclear is the impact that these differences might have on the cost and plausibility of CFXs when certification procedures are leveraged to generate CFXs.

\subsection{Generating Robust CFXs using \ourmethod}
\label{ssec:generating}

\begin{algorithm}[b]
\caption{Generation of Robust CFXs}\label{alg:CFX_gen}
\begin{algorithmic}[1]
\small
\STATE \textbf{Input:} FFNN $\model{}$, input $x$ such that $\model{(x)} = c$, set of plausible model shifts $\Delta$, maximum iteration number $\tau$  
\STATE \textbf{Output: } $\Delta$-robust CFX $x'$
\vspace{0.2cm}

\STATE $t \gets 0$ \hfill $\rhd$ iteration number
\WHILE{$t < \tau$ }
    \STATE $x' \gets$ \texttt{ComputeCFX}($x, \model{}$)
    \STATE rate $\gets \ourmethod(\model{},  x', \Delta)$
    \IF{$rate = 1$}
        \STATE \textbf{return} $x'$ \hfill $\rhd$ $x'$ is approx. $\Delta$-robust
    \ELSE
        \STATE increase allowed distance of next CFX 
        \STATE increase iteration number $t$
    \ENDIF
\ENDWHILE

\STATE \textbf{return} no robust CFX can be found

\end{algorithmic}
\end{algorithm}

The results discussed in the previous section have important implications on algorithms for the generation of robust CFXs. Recent works~\cite{Dutta_22,Jiang_Leofante_Rago_Toni_2023,Hammanetal23} have proposed iterative procedures that generate provably robust CFXs by alternating two phases. First, a CFX is generated solving (variations of) Definition~\ref{def:cfx}; then, a robustness certification procedure is invoked on the CFX. If the CFX is robust, then it is returned to the user; otherwise, the search continues, allowing for CFXs of increasing distance to be found. Clearly, the certification step has the potential to affect the CFXs computed in several ways. A robustness test that is too conservative may discard potentially good explanations and keep relaxing the distance constraint until the CFX is deemed robust. Ultimately, this may result in CFXs that exhibit poor proximity and plausibility.  

To test this hypothesis, we adapt the CFX generation algorithm of~\cite{Jiang_Leofante_Rago_Toni_2023} and replace their $\Delta$-robustness test with the one performed by \ourmethod. The complete procedure is shown in Algorithm~\ref{alg:CFX_gen}. 
In detail, after some initialization steps, we compute the first CFX using \texttt{ComputeCFX($x, \mathcal{M}$)} (line 5), which employs the solution proposed in \cite{Jiang_Leofante_Rago_Toni_2023} and presented above. Given a CFX $x'$ and a plausible model shift $\Delta$, at line 6, we employ \ourmethod setting $\alpha=0.999$ and $R=0.995$, thus obtaining $1378$ realizations to perform in the robustness test. If the CFX $x'$ returned by our approximation results robust for all these realizations, then we return it to the user. Otherwise, we increased the allowed distance for the next CFX generation and the iteration number $t$ (lines 10-11).

We then compare the resulting procedure with the four generation algorithm studied in~\cite{Jiang_Leofante_Rago_Toni_2023}: Wacht-R, Proto-R, MILP-R, and finally, ROAR~\cite{UpadhyayJL21}. Notably, ROAR is specifically designed to generate robust CFXs under plausible model shifts using average-case certification. 
Using the same datasets and training procedures of §~\ref{ssec:different_angles}, we generate $50$ CFXs for each dataset. We evaluate CFXs based on their proximity, measured by the $\ell_1$ distance, and plausibility, measured by the local outlier factor (\textbf{lof}) which determines if an instance is within the data manifold by quantifying the local data density \cite{BreunigKNS00} ($+1$ for inliers, $-1$ otherwise). We average $\ell_1$ and \textbf{lof} over the generated CFXs. We also report \textbf{VM1} and \textbf{VM2} for completeness.
The results obtained, which we report in Table~\ref{tab:results}, confirm our hypothesis. Indeed, \ourmethod produces the best results across all datasets, always generating CFXs with high plausibility and better proximity. Notably, \ourmethod outperforms ROAR as well, producing CFXs that retain a higher degree of validity after retraining.

%% file: imgs/INN.tex
\begin{tikzpicture}[scale=0.9, every node/.style={scale=0.8}]

  
  \node[circle,draw=black, minimum width=0.75cm,fill=maygreen25] (input_1) at (0.5,-1) {\Large$x_1$};

  \node[circle,draw=black, minimum width=0.75cm,fill=lila25] (hidden_1) at (2,0) {};
  \begin{scope}[xshift=2cm,scale=0.7]
        \relua
    \end{scope}
  
  \node[circle,draw=black, minimum width=0.75cm,fill=lila25] (hidden_2) at (2,-2) {};
  \begin{scope}[xshift=2cm, yshift=-2cm,scale=0.7]
        \relua
    \end{scope}
  
    \node[circle,draw=black, minimum width=0.75cm,fill=red25] (output_1) at (3.5,-1) 
  {\Large$y$};

  \draw[->] (input_1) edge node[above, xshift=-0.95cm]{\small{$[-0.48, -0.25]$}} (hidden_1);
  \draw[->] (input_1) edge node[below, xshift=-0.9cm, yshift=-0.1cm]{\small$[-0.99, -0.76]$} 
  (hidden_2);

  \draw[->] (hidden_1) edge node[above,xshift=0.95cm]{{\small$[-1.14, -0.91]$}} (output_1);

  \draw[->] (hidden_2) edge node[below, xshift=0.7cm, yshift=-0.1cm]{\small$[0.69, 0.93]$}  (output_1); 
  
\end{tikzpicture}

%% file: sections/table_results.tex
\begin{table*}[ht!]
    
    \centering
    
    \caption{Comparison on the robustness of CFXs using five state-of-the-art methods and \ourmethod proposed in this work.}
    \label{tab:results}
    \vspace{5mm}
    \resizebox{2\columnwidth}{!}{
    \begin{tabular}{ccccccccccccccccc}

        \toprule
        &
         \multicolumn{4}{c}{\textit{Diabetes} } &
         \multicolumn{4}{c}{\textit{no2} } &
         \multicolumn{4}{c}{\textit{SBA} } &
         \multicolumn{4}{c}{\!\!\!\! \textit{Credit} } \\
        
        \midrule
        & 
        \textbf{VM1}& 
        \!\!\!\textbf{VM2}\!\!\! & 
        \!\!\!$\ell_1$\!\!\! &
        \textbf{lof} &
        \textbf{VM1}& 
        \!\!\!\textbf{VM2}\!\!\! & 
        \!\!\!$\ell_1$\!\!\! &
        \textbf{lof} &
        \textbf{VM1}& 
        \!\!\!\textbf{VM2}\!\!\! & 
        \!\!\!$\ell_1$\!\!\! &
        \textbf{lof} &
        \textbf{VM1}& 
        \!\!\!\textbf{VM2}\!\!\! & 
        \!\!\!$\ell_1$\!\! &
        \textbf{lof} \\
        & 
        $\delta=0.11$ & 
        \!\!\!$\delta_e=0.27$\!\!\! & 
        \!\!\!\!\!\! &
         &
        $\delta=0.02$& 
        \!\!\!$\delta_e=0.07$\!\!\! & 
        \!\!\!\!\!\! &
         &
        $\delta=0.11$& 
        \!\!\!$\delta=0.25$\!\!\! & 
        \!\!\!\!\!\! &
        &
        $\delta=0.05$& 
        \!\!\!$\delta_e=1.28$\!\!\! & 
        \!\!\!\!\! &
        \\
        
        \midrule
        %
        \!\!\!\!Wacht-R\!\!\!\! & 
        100\% &
        \!\!\!100\%\!\!\! &
        \!\!\!0.122\!\!\! &
        1.00 &
        100\% &
        \!\!\!100\%\!\!\! &
        \!\!\!0.084\!\!\! &
        1.00 &
        92\%&
        \!\!\!92\%\!\!\!&
        \!\!\!0.023\!\!\!&
        -0.78&
        - &
        - &
        - &
        - \\
         \midrule
        %
        %
        Proto-R & 
        100\% &
        \!\!\!96\%\!\!\! &
        \!\!\!0.104\!\!\! &
        1.00 &
        100\%&
        \!\!\!100\%\!\!\! &
        \!\!\!0.069\!\!\! &
        1.00 &
        90\%&
        \!\!\!88\%\!\!\!&
        \!\!\!0.011\!\!\!&
        -0.02 &
        32\%&
        \!\!\!30\%\!\!\!&
        \!\!\!0.300\!\!\!&
        -1.00\\ 
        \midrule
        %
        %
        MILP-R & 
        100\% &
        \!\!\!100\%\!\!\! &
        \!\!\!0.212\!\!\! &
        -0.48 &
        100\%&
        \!\!\!100\%\!\!\! &
        \!\!\!0.059\!\!\! &
        1.00&
        100\%&
        \!\!\!100\%\!\!\!&
        \!\!\!0.018\!\!\!&
        -0.88 &
        100\%&
        \!\!\!100\%\!\!\!&
        \!\!\!0.031\!\!\!&
        1.00\\
         \midrule
        ROAR & 
        82\% &
        \!\!\!14\%\!\!\! &
        \!\!\!0.078\!\!\! &
        0.95 &
        88\%&
        \!\!\!34\%\!\!\! &
        \!\!\!0.074\!\!\! &
        1.00&
        82\%&
        \!\!\!78\%\!\!\!&
        \!\!\!0.031\!\!\!&
        -0.80&
        62\%&
        \!\!\!60\%\!\!\!&
        \!\!\!0.047\!\!\!&
        1.00\\
         \midrule
         \textbf{\ourmethod} & 
        100\% &
        \!\!\!100\%\!\!\! &
        \!\!\!0.077\!\!\! &
        1.00&
        100\%&
        \!\!\!100\%\!\!\! &
        \!\!\!0.042\!\!\! &
        1.00&
        100\%&
        \!\!\!100\%\!\!\!&
        \!\!\!0.009\!\!\!&
        0.44&
        100\%&
        \!\!\!94\%\!\!\!&
        \!\!\!0.028\!\!\!&
        1.00\\
        \bottomrule

    \end{tabular}
    }
  
    \vspace{5mm}
\end{table*}

%% file: sections/conclusion.tex
\section{Conclusions}
We studied the problem of robustness of CFXs with respect to plausible model shifts. We proved for the first time that certifying the robustness of CFX with respect to this notion is an NP-complete problem. This results motivates the quest of new scalable algorithm to certify robustness under PMS. We then compared with existing methods to generate robust CFXs with probabilistic guarantees and showed that these approaches may not be directly applicable to the PMS setting. We then introduced \ourmethod, a novel scalable approach for probabilistic robustness certification, and used it to generate robust CFXs under model shifts. Our results demonstrate the advantages of our approach, outperforming SOTA methods on a range of metrics. 
This paper opens several avenues for future work. Firstly, while our experiments only focused on DNNs, there seems to be no reason why \ourmethod could not be applied to other parameterized models for which the notion of plausible model changes holds. Secondly, it would be interesting to investigate improvements to the approximation scheme presented here to further tighten the robustness guarantees our framework can offer. Finally, it would be interesting to apply \ourmethod to other notions of robustness studied in the literature, such as robustness to noisy execution.

%% file: sections/appendix-proofs.tex
\clearpage
\appendix

\section{Supplementary Material}

\subsection{The proof of Lemma \ref{lem:diff}}

\noindent
{\bf Lemma \ref{lem:diff}.}
{\em     Naturally-occurring model shifts may not be plausible, and vice-versa.}
\begin{proof}
We showed already that there are examples of \textit{plausible model shifts} that cannot be characterized as \textit{naturally occurring model changes}. Here we complete the argument. 

Although the existence of \textit{naturally occurring model shifts} not being plausible is implicit in the definition, for the sake of completeness, 
we provide an example network.

Consider a DNN having a single
input value $x$ and a single parameter $\theta$ and computing the function 
$$\model{\theta}(x) = ReLU(0.5 - ReLU (x - \theta))$$  

  \begin{figure}[h!]
        \centering
        \scalebox{0.9}{\input{imgs/net_supp}}
        \vspace{3mm}
        \caption{The DNN considered in this proof}
        \label{fig:proofA1}
        \vspace{5mm}
    \end{figure}
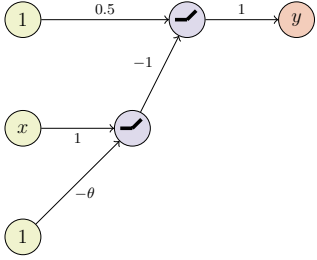

Fix a data set ${\cal X}$ and let $\theta = \max_{x \in {\cal X}} x.$ 
Let us consider the set of model shifts 
$\Sigma = \{S_{\tau}\mid \tau \in \mathbb{R}_+\}$ defined by  $S_{\tau}(\model{\theta}) = \model{\theta+\tau}.$
Clearly for any $\tau \geq 0$, we have 
$$\model{\tau+\theta}(x) = \model{\theta}(x) = 0.5,$$ for any $x \in {\cal X}.$
This trivially implies that $\Sigma$ is a set of naturally occurring model shifts (all shifts considered have exactly the same value in all points in ${\cal X}$).

The claim now follows by observing that there is no finite $\delta$ such that the corresponding set of plausible model shifts 
$\Delta = \{ \mshift \mid \dist{p}{\model{\theta}}{\mshift(\model{\theta})} \leq \delta\}$
contains $\Sigma.$


\end{proof}

\subsection{Plausible model shifts vs naturally-occurring model shifts: additional results}

\input{sections/table_compare}

We considered three binary classification datasets commonly used in XAI:
\begin{itemize}
    \item the \textit{credit} dataset~\cite{Dua2019}, which is used to predict the credit risk of a person (good or bad) based on a set of attribute describing their credit history;
    \item the \textit{spambase} dataset~\cite{misc_spambase_94} is used to predict whether an email is to be considered spam or not based on selected attributes of the email;
    \item the \textit{online news popularity} dataset~\cite{misc_online_news_popularity_332}, referred to as \textit{news} in the following, is used to predict the popularity of online articles.
\end{itemize}

We trained a neural network classifier with two hidden layers (20 and 10 neurons, respectively) for each dataset and used a Nearest-Neighbor Counterfactual Explainer~\cite{guidotti2022counterfactual} to generate counterfactual explanations for $10$ different inputs. After generating a counterfactual, we produce $n$ different perturbations $\model{{\theta'}}$ of the original neural network $\model{{\theta}}$ for $n \in \{1000, 10000\}$ under \textit{plausible} model shift with $\Delta \in \{0.05, 0.1, 0.2, 0.3\}$. 
We then considered two measures:
\begin{itemize}
    \item average difference in output between $\model{{\theta}}$ and $\model{{\theta'}}$, for each of the $n$ model $\model{{\theta'}}$ and across all CFXs;
    \item for each counterfactual, we perform a one-sided t-test to check whether the average prediction generated by $n$ models $\model{{\theta'}}$ equals the original prediction of $\model{{\theta}}$. We report the percentage of CFXs for which the null hypothesis was rejected (p-value used $0.05$).
\end{itemize}

Table~\ref{tab:compare} reports our results. We observe that the requirement that the expected output of perturbed models remains equal to the original prediction is often violated. This complements the result of Lemma~\ref{lem:diff}, confirming that the two notions of model changes are orthogonal. In the following we focus on certification approaches for robustness under plausible model shifts.

\subsection{The proof of Theorem \ref{th:hardness-CFX}}

\noindent{\bf Theorem \ref{th:hardness-CFX}.}
{\em     {\sc DRPIA} is NP-complete.}

\begin{proof}
    We can show that {\sc 3-SAT} reduces to {\sc DRPIA}.
    In particular, given a formula $\phi$ we can produce an DNN $\model{}$, an input $x$ and 
    a shifts on the edges, defining a INN $\intervalnet$ such that $\phi$ is satisfiable 
    if and only if there exists another DNN $\model{}'$ which is also a realization
    of $\intervalnet$ such that $0.5 \leq \model{}'(x) $ and  $\model{}(x) < 0.5.$

The INN
$\intervalnet$ is an {\em interval abstraction under a set of plausible shifts $\Delta$.} Precisely, 
this means that the reduction produces an INN where all intervals on the edges are 
of the same width $2\delta$

\begin{figure}[h!]
        
        \centering
        \includegraphics[width=0.8\linewidth]{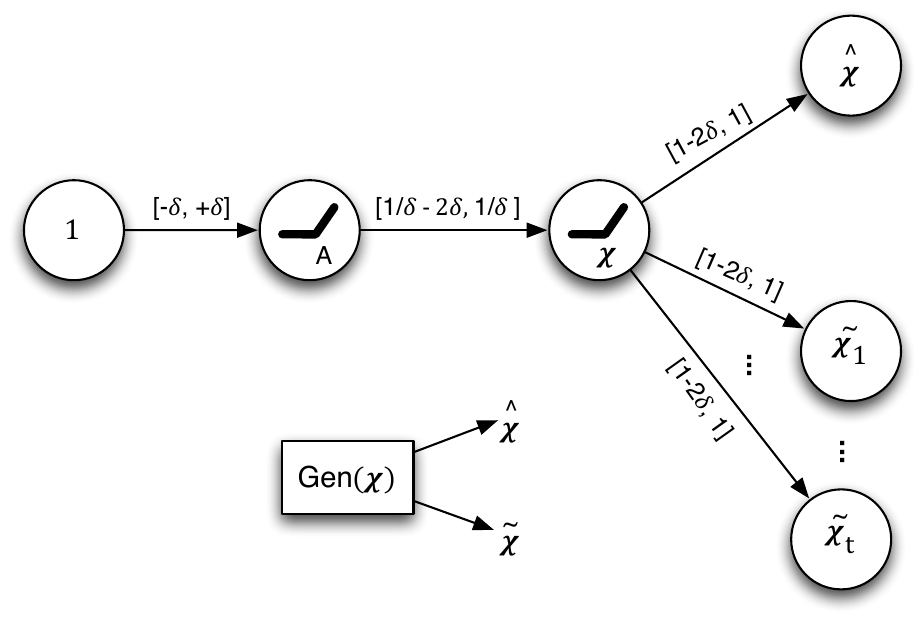}
        \vspace{3mm}
        \caption{ Generating-gadget }
        \label{fig:var-gen}
        \vspace{5mm}
    \end{figure}

\begin{figure}[h!]
        
        \centering
        \includegraphics[width=0.75\linewidth]{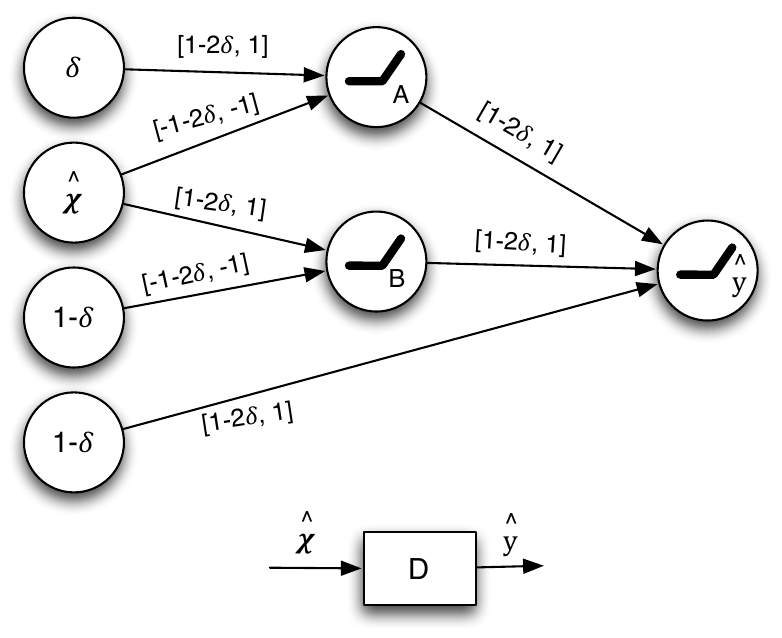}
        \vspace{3mm}
        \caption{Discretizer.}
        \label{fig:Discretizer}
        \vspace{5mm}
    \end{figure}

The following lemmas (Lemmas \ref{Discretizer}-\ref{end-gadget}) are easily verified by direct inspection of the possible output values of the 
nodes in the respective network-gadgets they refer too.  They provide the key properties of the 
gadgets we are using for the reduction.

\begin{lemma}[Discretizer-gadget] \label{Discretizer}
For the {\em Discretizer}-gadget in figure \ref{fig:Discretizer} the following holds: We have $\hat{y}=1$ if and only if the input value $\hat{\chi}$ is binary, i.e., $\hat{\chi} \in \{0,1\}.$

In addition, if $\hat{\chi}$ is the corresponding output of a {\em Generating}-gadget with $t+1$ many outputs, $\hat{\chi}, \tilde{\chi}_1, \dots, \tilde{\chi}_t,$ (see Fig. \ref{fig:var-gen}), then 
$\hat{y} = 1$ iff for all $i \in [t]$, $\tilde{\chi}_i\in \{0\} \cup [1-2\delta, 1].$ In particular,  $\hat{y} = 1$ if and only if either $\tilde{\chi}_i = 0$ and $\hat{\chi} = 0$ or $\tilde{\chi}_i \in [1-2\delta, 1]$ and $\hat{\chi} = 1.$
\end{lemma}

\begin{figure}[h!]
        
        \centering
        \includegraphics[width=\linewidth]{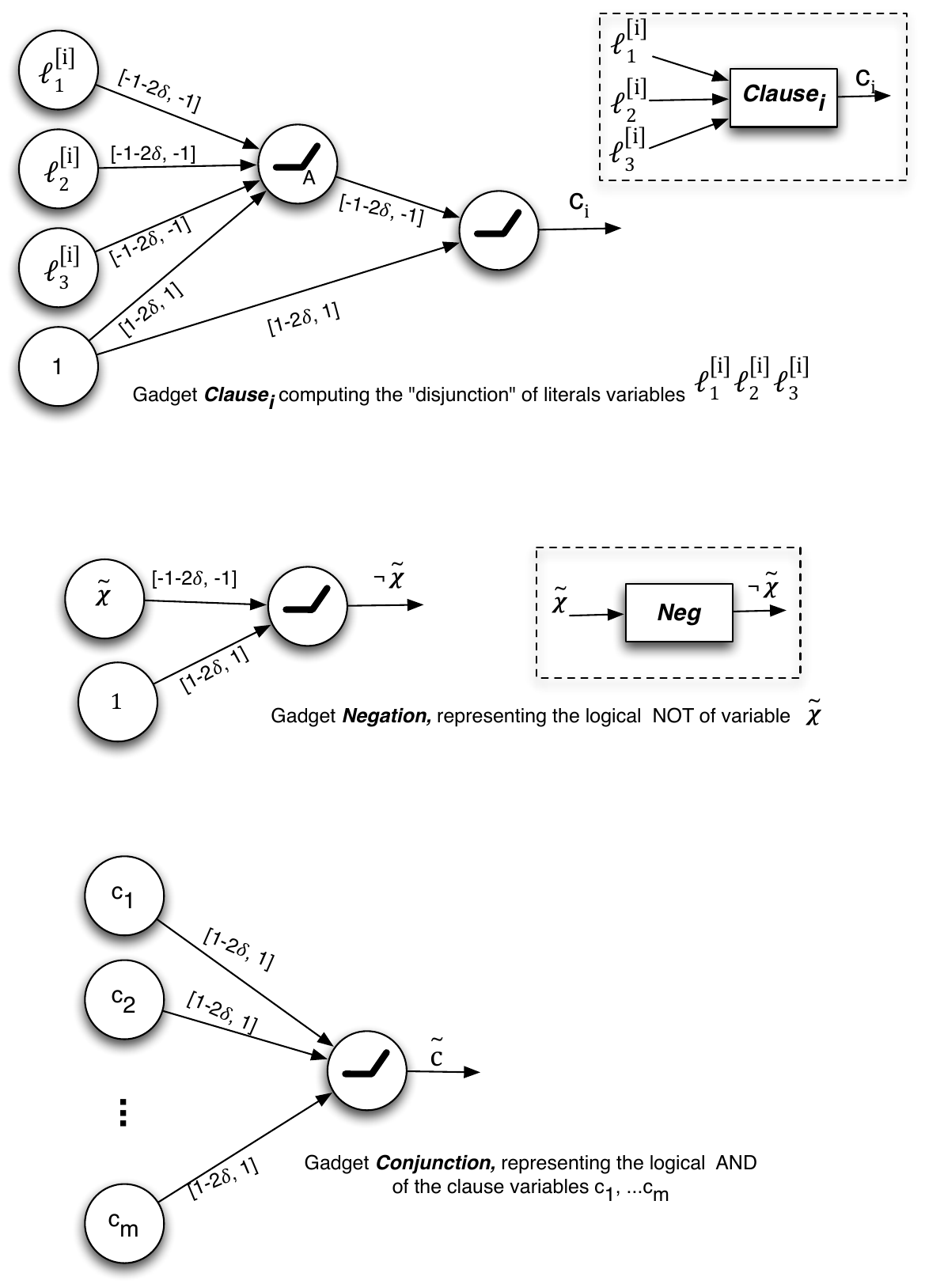}
        \vspace{3mm}
        \caption{LOGICALPORTS}
        \label{fig:LogicalPorts}
        \vspace{5mm}
    \end{figure}

\begin{lemma}[Negation-gadget] \label{not-gadget}
With reference to the {\em Negation}-gadget in Fig.\ref{fig:LogicalPorts},
the following holds. Assume $\hat{\chi} \in \{0\} \cup [1-2\delta, 1].$ Then, $\neg \tilde{\chi} \in [0, 2\delta] \cup [1-2\delta, 1].$ In particular, we have that $\neg\tilde{\chi} \in [0, 2\delta]$ iff 
$\tilde{\chi} \in [1-2\delta, 1]$ and 
$\neg\tilde{\chi} \in  [1-2\delta, 1]$ iff 
$\tilde{\chi} = 0$
\end{lemma}

\begin{lemma}[Clause-gadget] \label{clause-gadget}
With reference to the {\em Clause}-gadget in Fig.\ref{fig:LogicalPorts},
the following holds: Assume 
$\ell^{[i]}_1, \ell^{[i]}_2, \ell^{[i]}_3 \in [0,2\delta]\cup[1-2\delta, 1].$ 
Then, 
\begin{itemize}
    \item $c_i \in [0, 8\delta+2\delta^2]$ iff 
    $\ell^{[i]}_1, \ell^{[i]}_2, \ell^{[i]}_3 \in [0,2\delta]$;
    \item $c_i \in [1- 4\delta-4\delta^2, 1]$ iff there is $t\in\{1,2,3\}$
    such that $\ell^{[i]}_t \in [1-2\delta,1]$;
\end{itemize} 
\end{lemma}

\begin{lemma}[Conjunction-gadget] \label{conjunction-gadget}
With reference to the {\em Conjunction}-gadget in Fig.\ref{fig:LogicalPorts},
the following holds. Assume that for each $j=1, \dots, m,$ it holds that
$c_j \in [0, 8\delta+2\delta^2] \cup [1- 4\delta-4\delta^2].$
If $\delta < 2/25$ then 
\begin{itemize}
    \item $8\delta+2\delta^2 < 1-4\delta - 4\delta^2$;
    \item $\tilde{c} = m$ iff $c_i = 1$ for each $i=1, \dots, m$;
    \item $\tilde{c} \leq (m-1)+8\delta + 2\delta^2$ iff there exists $j \in \{1, \dots, m$ such that 
    $c_j \in [0, 8\delta+2\delta^2$
\end{itemize} 
\end{lemma}

\begin{figure}[h!]
        
        \centering
        \includegraphics[width=0.5\linewidth]{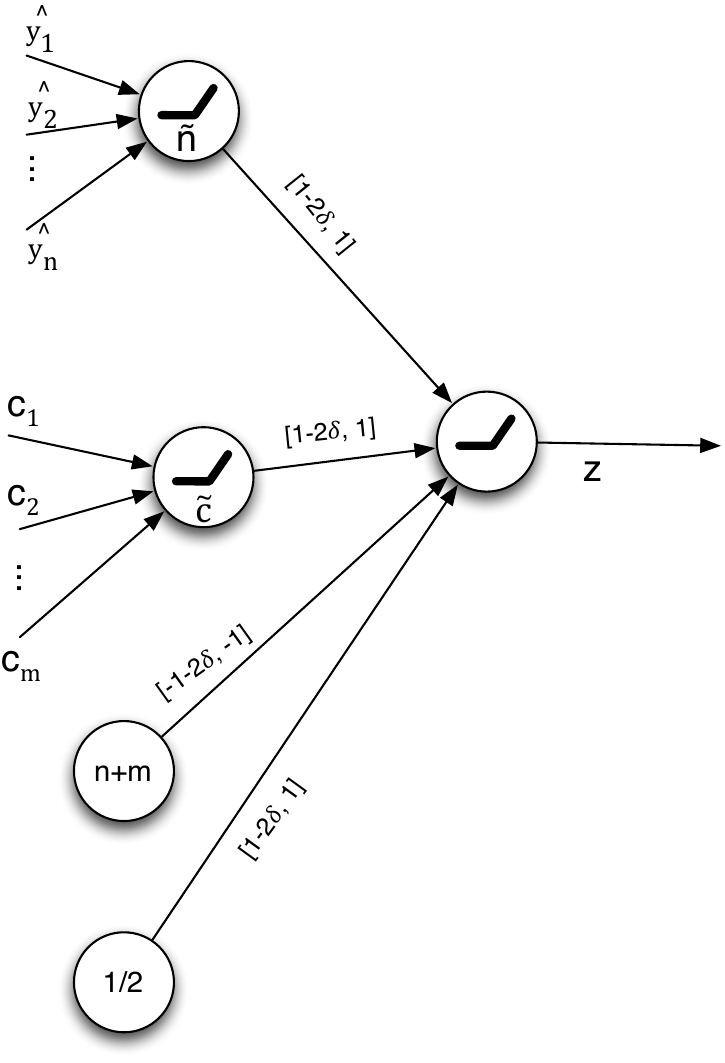}
        \vspace{3mm}
        \caption{End-gadget}
        \label{fig:End-gadget}
         \vspace{5mm}
    \end{figure}

\begin{lemma}[End-gadget] \label{end-gadget}
With reference to the {\em End}-gadget in Fig.\ \ref{fig:End-gadget}) we have the following:
$z = 1/2$ iff $\tilde{n} = n$ and $\tilde{c} = m,$ i.e. (using also Lemma \ref{clause-gadget}
\begin{itemize}
    \item for each $j=1, \dots, n$ it holds that $\hat{y}_j= 1$;
    \item for each $i=1, \dots, m$ it holds that $c_i= 1$;
\end{itemize}
\end{lemma}

\begin{figure}[h!]
        
        \centering
        \includegraphics[width=\linewidth]{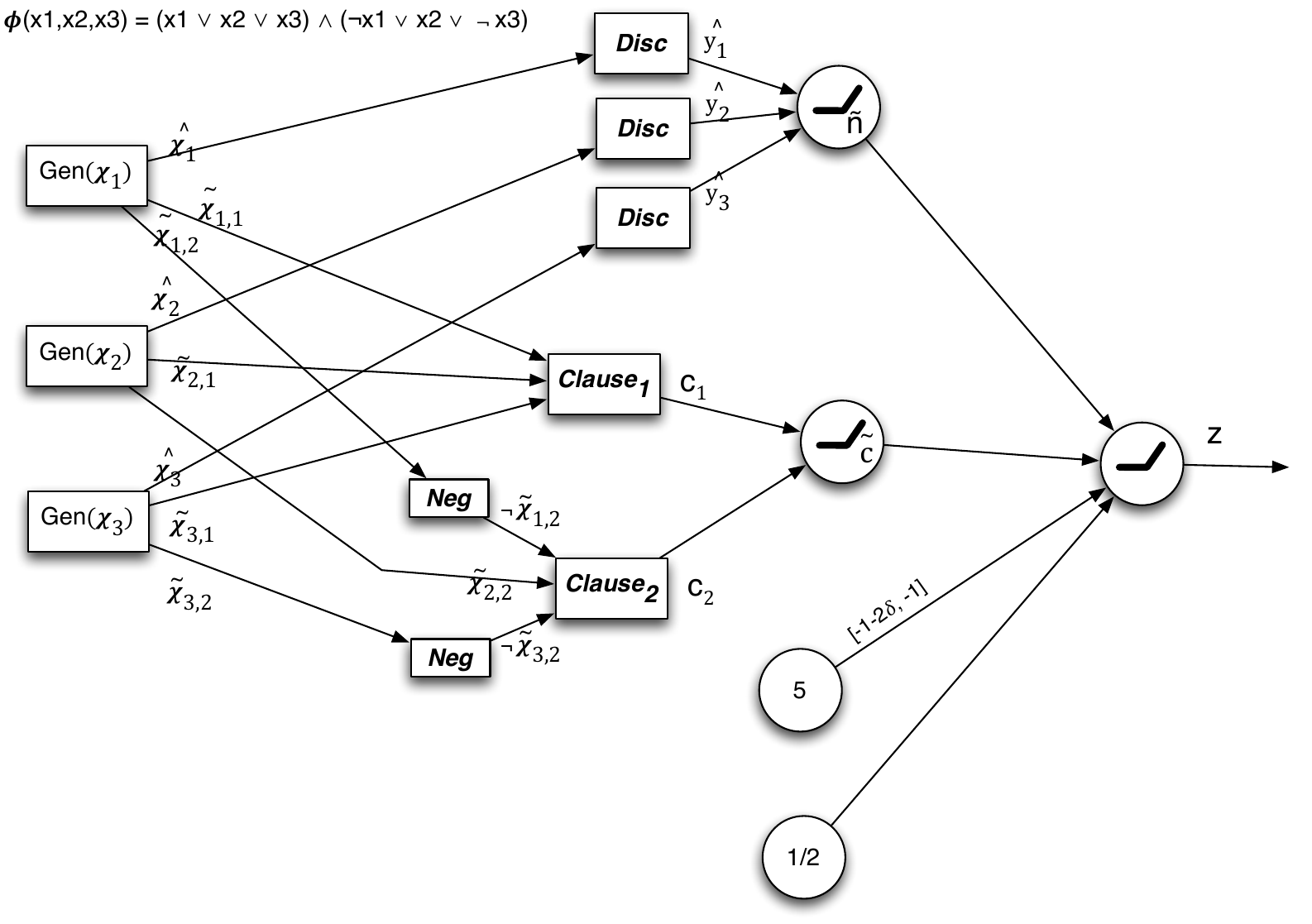}
         \vspace{3mm}
\caption{A complete Example of the reduction on a simple formula, with $n = 3$ variables and $m=2$ clauses. All the 
interval weights not explicitly given are  $[1-2\delta, 1]$}
        \label{fig:Formula-Example}
         \vspace{5mm}
    \end{figure}

\noindent 
{\bf Final part of the argument.}
Fix a 3-CNF $\phi(x_1, \dots, x_n).$ 
Consider the INN built as in the example in Fig. \ref{fig:Formula-Example}:
for each variable $x_i$ add to the network  a {\em Generating}-gadget with one output $\hat{\chi}_i$ and as many outputs $\tilde{\chi}_i$ as the number of 
occurrences of $x_i$ in the formula $\phi$ (in positive or negated form).

Send each output $\hat{\chi}_i$ to a distinct copy of the 
{\em Discretizer}-gadget, and refer to the corresponding output as $\hat{y}_i.$

For each clause $C_j = (\lambda^{(j)}_1 \vee \lambda^{(j)}_2 \vee \lambda^{(j)}_3)$ ($j=1, \dots, m$) of $\phi$ add a {\em Clause}-gadget, henceforth referred to as $Clause_j.$ 
For each $t = 1, 2, 3,$ 
\begin{itemize}
    \item if $\lambda^{(j)}_t$ corresponds to the positive variable $x_i$ then connect the input of $Clause_j$ labelled $\ell^{(j)}_t$ to a distinct output $\tilde{\chi}_{i,r}$ of the $i$th Generating-gadget.
    \item if $\lambda^{(j)}_t$ corresponds to the negated variable $\neg x_i$ then connect the input of $Clause_j$ labelled $\ell^{(j)}_t$ to the output of a negation gadget and the input of such negation gadget to a distinct output $\tilde{\chi}_{i,r}$ of the $i$th Generating-gadget.
\end{itemize}

Connect all the outputs $\hat{y}_i$ ($i=1, \dots, n$) of the Clause-gadget to a RELU node whose output we refer to as $\tilde{n}.$ 

Connect all the outputs $c_j$ ($j=1, \dots, m$) of the Clause-gadget to a RELU node (as in a Conjunction gadget) whose output we refer to as $\tilde{c}.$

Finally, connect $\tilde{n}, \tilde{c}$ as in the End-gadget (Fig. \ref{fig:End-gadget}).


\bigskip

\noindent
{\bf If $\phi$ is satisfiable then there is a choice of the parameters (weights) 
such that 
$z = 1/2.$}

Let  ${\bf a} = (a_1, \dots, a_n)$ be an assignment such that\footnote{For the ease of the presentation, we are identifying the logical value {\tt true} with the boolean value 1 and the 
logical value {\tt false} with the Boolean value 0, both for the variables and the value computed by $\phi$.} 
$\phi({\bf a}) = 1$.
For each $i=1, \dots, n,$ 
\begin{itemize} 
\item if $a_i = 1$ then choose the first two weights in the {Generating}-gadget of Fig. \ref{fig:var-gen} to be respectively $\delta$ and $1/\delta,$ so that we have that $\chi_i = 1$
\item if $a_i = 0$ then choose the first two weights in the {Generating}-gadget of Fig. \ref{fig:var-gen} to be respectively $0$ and $1/\delta,$ so that we have that $\chi_i = 0$
\end{itemize}

Choose the remaining two weights of the generating gadget to be $1$, hence $\hat{\chi}_i = \tilde{\chi}_i = a_i.$

Set  the weights in the rest of the network to $1$ or $-1$ according to whether the interval they are taken from is $[1-2\delta, 1]$ or $[-1-2\delta, -1],$ respectively.

Then, it is not hard to verify the following values computed by the nodes of the network. 
\begin{enumerate}
\item $\hat{y}_i = 1$ for each $j=1, \dots, n$ and $\tilde{n} = n.$ \label{pt1}.
\item for each $i=1, \dots, m,$ the output $c_i$ of the {\em Clause}-gadget corresponding to the $i$-th clause of $\phi$ will have value $c_i = 1.$ For this, it is sufficient to notice that, with the chosen weights, each value  $\ell_t^{[i]}$ is equal to the value that the assignment ${\bf a}$ induces on the  literal  
$\lambda^{(i)}_t.$ By assumption on the assignment ${\bf a}$ there exists a $t$ such that $\lambda^{(i)}_t = 1,$ hence $\ell_t^{[i]} = 1$ which together with the choice of the weights in the corresponding conjunction gadget, implies $c_i = 1.$ 
\item $\tilde{c} = m$ (following from the previous point) \label{pt4}.
\item $z = \tilde{n} + \tilde{c} - (n+m) + 1/2 = 1/2$ (from points \ref{pt1} and \ref{pt4}). 
\end{enumerate}

\bigskip

\noindent
{\bf If there exists a choice of the weights such that $z = 1/2$ then $\phi$ is satisfiable.}

By Lemma \ref{end-gadget}, it must hold true that 
$\tilde{c} = m$ and $\tilde{n} = n$ 
whence 
\begin{enumerate}
    \item for each $i=1, \dots, m$ $c_i = 1$
    \item for each $j=1, \dots, n$ $\hat{y}_j = 1$
\end{enumerate}

By Lemma \ref{Discretizer}, because of item 2, we must have that for each $j=1, \dots, n$ it holds that
$\hat{\chi}_j \in \{0,1\},$ and in particular, 
\begin{itemize}
    \item for each 
$j$ such that $\hat{\chi}_j = 0$ we have that for each $s,\, \tilde{\chi}_{j,s} = 0$;
    \item for each $j$ such that $\hat{\chi}_j = 1 $ we have that for each $s,\, \tilde{\chi}_{j,s} \in [1-2\delta, 1]$;
\end{itemize} 

Therefore, all the literal values $\ell^{[i]}_t$ satisfy the assumption on the domain in the statement of Lemma \ref{clause-gadget}. For this, notice that the literal values are either equal to some $\tilde{\chi}_j$ or to the output of a {\em Negation}-gadget with input $\tilde{\chi}_j.$ Hence the claim holds by using also Lemma \ref{not-gadget}.

Since $c_i = 1$, by Lemma \ref{clause-gadget} we have that 
there exists $t \in \{1, 2, 3\}$ such that $\ell^{[i]}_t \in \{1-2\delta, 1\}.$

We now define the assignment ${\bf a} = (a_1, \dots, a_n)$ by setting $a_j = \hat{\chi}_j,$
for each $j=1, \dots, n.$

We claim that $\phi({\bf a}) = 1.$

Suppose, by contradiction, that there exists some clause $C_i$ for which the assignment ${\bf a}$ makes all the literals equal to $0.$

By hypothesis, we have that the activation value $c_i = 1.$ Hence, by Lemma \ref{clause-gadget} there exists 
$t \in \{1,2,3\}$ such that
$\ell^{[i]}_t \in [1-2\delta, 1].$ Let us now consider the corresponding literal $\lambda^{(i)}_t$ in clause $C_i.$ We are going to show that the value induced on it by the assignment ${\bf a}$ must be 1.

\noindent
{\em Case 1.} $\lambda^{(i)}_t = x_j$ (for some $j=1, \dots, n$), i.e., the literal corresponds to a non-negated variable. 

Then, by construction, there is an $r$ such that $\ell^{[i]}_t$ is the output $\tilde{\chi}_{j,r}$ of the generating gadget associated to variable $x_j.$
By Lemma \ref{Discretizer}, since
$\hat{y}_j = 1$ and $\tilde{\chi}_{j,r} \in [1-2\delta, 1]$ it holds that, with the assignment ${\bf a}$ we have $\hat{\chi}_j = 1,$ whence 
$\lambda^{(i)}_t = x_j = a_j = \hat{\chi}_j = 1$ contradicting the absurdum hypothesis that all literals of $C_i$ are 0 (false).

\medskip

\noindent 
{\em Case 2.} $\lambda^{(i)}_t = \neg x_j$ (for some $j=1, \dots, n$), i.e., the literal corresponds to a negated variable. 

Then, by construction, there is an $r$ such that $\ell^{[i]}_t$ is the output 
of a negation-gadget whose input is  
$\tilde{\chi}_{j,r}$ of the generating gadget associated to variable $x_j.$

Since $\hat{y}_j = 1$ we have that  $\tilde{\chi}_{j,r} \in \{0\} \cup [1-2\delta, 1].$ Hence by Lemma \ref{not-gadget}, since the output $\ell^{[i]}_t \in [1-2\delta, 0],$ we must have that the input
$\tilde{\chi}_{j,r} = 0.$ It follows, by 
Lemma \ref{Discretizer} that $a_j = \hat{\chi}_j = 0.$ Therefore, with the assignment ${\bf a},$ we have
$\lambda^{(i)}_t = \neg x_j = (1- a_j) = 1,$ contradicting the absurdum hypothesis that all literals of $C_i$ are assigned a value 0 (false) by ${\bf a}$.

Since in either case we reach a contradiction, we have shown that all
clauses are satisfied by the assignment ${\bf a},$ i.e., $\phi({\bf a}) = 1.$

The proof is complete. 
\end{proof}

\section{Exact enumeration method}

As sketched in Section~\ref{ssec:soundness}, a possible approach to exactly enumerate the portion of an interval abstraction for which the CFX $x'$ is not robust is to recursively split in half each INN's interval into two equal parts until each INN resulting from the split allows to determine that $x'$ is entirely robust or non-robust. Algorithm~\ref{alg:exact} below formalizes these ideas.

In lines 3-5 we initialize data structures needed to keep track on the portion of INNs that are robust, non-robust and yet to be decided. We iteratively compute the output reachable set of the INN and check whether the robustness condition is met (lines 9-11) or violated (lines 12-14) and remove the corresponding INN from the stack of undecided INNs. Otherwise, we choose an interval in $\intervalnet$ and split it into half to create two new INNs, which we add to the stack. We repeat this procedure until the stack of undecided INNs is empty or an threshold $\epsilon$ on the precision of this procedure is met. Once the loop terminates, we return the set of INNs for which the input CFX is robust.

\begin{algorithm}[h!]
\caption{\textit{Exact CFX $\Delta$-Robustness}}\label{alg:exact}
\begin{algorithmic}[1]
\small
\STATE \textbf{Input:} An INN $\mathcal{N}$ and a CFX $x'$ and an maximum $\epsilon$-precision for the splitting phase
\STATE \textbf{Output: } set of INNs for which $x'$ is robust.
\vspace{0.1cm}

\STATE robust\_INNs $\gets \emptyset$ 
\STATE non-robust\_INNs $\gets \emptyset$ 
\STATE unknown $\gets \;\texttt{Push}(\mathcal{N})$

\vspace{0.1cm}

\WHILE{(unknown $\neq \emptyset$) or $(\epsilon$-precision not reached)}
    \STATE $\intervalnet \gets \texttt{GetINNToVerify}$(unknown)
    \STATE  $\mathcal{R}_{\intervalnet} \gets \texttt{ComputeReachableSet}(\intervalnet,\; x')$
    \IF{$\texttt{lower}(\mathcal{R}_{\intervalnet}) \geq 0.5$}
        \STATE robust\_INNs $\gets\;\texttt{Push}(I)$ 
        \STATE unknown $\gets \texttt{Pop}(\intervalnet)$
    \ELSIF{$\texttt{upper}(\mathcal{R}_{\intervalnet}) < 0.5$}
        \STATE non-robust\_INNs $\gets\;\texttt{Push}(\intervalnet)$ 
        \STATE unknown $\gets \texttt{Pop}(\intervalnet)$
    \ELSE
        \STATE $\intervalnet', \intervalnet'' \gets \texttt{ChooseIntervalToSplit}(\intervalnet)$
        \STATE unknown $\gets \;\texttt{Push}(\intervalnet', \intervalnet')$   
    \ENDIF   
\ENDWHILE

\STATE \textbf{return} robust\_INNs
\end{algorithmic}

\end{algorithm}

\newpage
\section{Computing a provable $\Delta_{max}$}

Algorithm~\ref{alg:new_approach} can be modified to compute the maximum admissible shift under worst-case guarantees. In a nutshell, this simply requires replacing the robustness test performed by \ourmethod with the MILP-based certification procedure. For completeness we report the resulting procedure in Algorithm~\ref{alg:MILP_binary}.

\begin{algorithm}[h!]
\caption{Provable Plausible $\Delta$-Shift}\label{alg:MILP_binary}
\begin{algorithmic}[1]
\small
\STATE \textbf{Input:} Model $\model{\theta}$, CFX $x'$, $\alpha$, $R$ 
\STATE \textbf{Output: } $\delta_{max}$
\vspace{0.1cm}

\STATE $\delta_{init} \gets 0.0001$
\STATE rate $\gets \texttt{MILP}(\model{\theta}, x', \delta_{init})$
 \IF{rate $\neq 1$}
    \STATE \textbf{return} $0$ \hfill $\rhd$ no robustness
\ENDIF
\vspace{0.1cm}
\STATE $\delta \gets \delta_{init}$
\WHILE{rate $= 1$}
    \STATE $\delta \gets 2\delta$
    \STATE rate $\gets \texttt{MILP}(\model{\theta},  x', \delta)$
\ENDWHILE

\STATE $\delta_{max} \gets \delta/2$

\WHILE{True}
    \IF{$\vert \delta - \delta_{max}\vert \leq \delta_{init}$}
        \STATE \textbf{return} $\delta_{max}$
    \ENDIF
    \STATE $\delta_{new} \gets (\delta_{max} + \delta)/2$
    \STATE rate $\gets \texttt{MILP}(\model{\theta}, x', \delta_{new})$
    \IF{rate $= 1$}
        \STATE $\delta_{max} \gets \delta_{new}$
    \ELSE
        \STATE $\delta \gets \delta_{new}$
    \ENDIF
\ENDWHILE
\end{algorithmic}
\end{algorithm}

%% file: imgs/net_supp.tex
\begin{tikzpicture}[scale=0.8, every node/.style={scale=0.7}]

  
  \node[circle,draw=black, minimum width=0.75cm,fill=maygreen25] (input_1) at (0,0) {\Large$1$};
  \node[circle,draw=black, minimum width=0.75cm,fill=maygreen25] (input_2) at (0,-2) {\Large$x$};
\node[circle,draw=black, minimum width=0.75cm,fill=maygreen25] (input_3) at (0,-4) 
{\Large$1$};
\node[circle,draw=black, minimum width=0.75cm,fill=red25] (output) at (5,0) 
{\Large$y$};
  
  \node[circle,draw=black, minimum width=0.75cm,fill=lila25] (hidden_1) at (3,0) {};
  \begin{scope}[xshift=3cm,scale=0.7]
        \relua
    \end{scope}
  
  \node[circle,draw=black, minimum width=0.75cm,fill=lila25] (hidden_2) at (2,-2) {};
  \begin{scope}[xshift=2cm, yshift=-2cm,scale=0.7]
        \relua
    \end{scope}

  \draw[->] (input_1) edge node[above]{{$0.5$}} (hidden_1);

  \draw[->] (input_2) edge node[below]{$1$} 
  (hidden_2);

  \draw[->] (input_3) edge node[below, xshift=0.15cm]{$-\theta$} 
  (hidden_2);
  
  \draw[->] (hidden_2) edge node[above, xshift=-0.35cm]{{$-1$}} (hidden_1);

  \draw[->] (hidden_1) edge node[above]{$1$}  (output); 
  
\end{tikzpicture}

%% file: sections/table_compare.tex
\begin{table*}[t!]
    \centering
    \caption{Empirical evaluation across model perturbations of increasing magnitude $\delta$ and different sample sizes $n$.}
    \label{tab:compare}
    \vspace{5mm}
    \resizebox{\textwidth}{!}{
    \begin{tabular}{lcccccccccccc}
        \toprule &
        \multicolumn{4}{c}{\textbf{\textit{Credit}}} &
        \multicolumn{4}{c}{\textbf{\textit{Spam}}} &
        \multicolumn{4}{c}{\textbf{\textit{News}}} \\
        \toprule &
         \multicolumn{2}{c}{$n = 1000$} & \multicolumn{2}{c}{$n = 10000$} & \multicolumn{2}{c}{$n = 1000$} & \multicolumn{2}{c}{$n = 10000$} & \multicolumn{2}{c}{$n = 1000$} & \multicolumn{2}{c}{$n = 10000$}\\
        \cline{2-13} 
        &\vspace{-2mm}&  && & &  &&  & &  & \\
        &Avg diff. & Rej. (\%) & Avg diff. & Rej. (\%) & Avg diff. & Rej. (\%) & Avg diff. & Rej. (\%) & Avg diff. & Rej. (\%) & Avg diff. & Rej. (\%)\\ \midrule
        $\delta = 0.05$ & 0.008 & 90 & 0.022 & 90 & 0.018 & 50 & 0.017 & 70 & 0.034 & 70 & 0.033 & 80\\ \midrule
        $\delta = 0.1$ & 0.017 & 100 & 0.047 & 100 & 0.034 & 100 & 0.035 & 100 & 0.064 & 80 & 0.063& 100\\ \midrule
        $\delta = 0.2$ & 0.046 & 100 & 0.086 & 100 & 0.0748 & 90 & 0.064 & 100 & 0.127 & 90 & 0.141& 100\\ \midrule
        $\delta = 0.3$ & 0.110 & 100 & 0.140 & 90 & 0.121 & 100 & 0.087 & 100 & 0.207 & 90 & 0.173 & 100\\ \bottomrule
    \end{tabular}
    }

\end{table*}